\documentclass{article}

\usepackage[nonatbib,main,final]{neurips_2026}

\usepackage{graphicx}
\usepackage[utf8]{inputenc} 
\usepackage[T1]{fontenc}    
\usepackage{hyperref}       
\usepackage{url}            
\usepackage{booktabs}       
\usepackage{amsfonts}       
\usepackage{amsmath}
\usepackage{amsthm}
\usepackage{enumitem}
\usepackage{mathtools}
\usepackage{nicefrac}       
\usepackage{microtype}      
\usepackage{xcolor}         
\usepackage{hyperref}
\usepackage{cleveref}
\usepackage{appendix}
\usepackage{algorithm}
\usepackage[noend]{algpseudocode}
\usepackage[
backend=biber,
style=numeric,
sorting=none
]{biblatex}
\usepackage[dvipsnames]{xcolor}
\usepackage[textsize=tiny]{todonotes}

\newcommand{\x}{\mathbf{x}}

\newcommand{\upd}{\textup{d}}
\newcommand{\la}{\langle}
\newcommand{\ra}{\rangle}

\newtheorem{condition}{Condition}

\newtheorem{proposition}{Proposition}

\theoremstyle{remark}
\newtheorem{remark}{Remark}
\Crefname{condition}{Condition}{Conditions}

\title{Stein Kernelized Molecular Dynamics for Active Learning of Interatomic Potentials}

\author{%
  Joanna Zou\textsuperscript{1,}\thanks{Correspondence to \texttt{jjzou@mit.edu}.} \quad
  Fraser Birks\textsuperscript{2,3} \quad
  Dallas Foster\textsuperscript{4} \quad
  Youssef Marzouk\textsuperscript{1}
  \\[1ex]
  \textsuperscript{1}Center for Computational Science \& Engineering, Schwarzman College of Computing, MIT
  \\
  \textsuperscript{2}Warwick Centre for Predictive Modelling, School of Engineering, University of Warwick
  \\
  \textsuperscript{3}Machine Learning Molecular Dynamics Group, Max Planck Institute for Polymer Research
  \\
  \textsuperscript{4}NVIDIA 
  \AND
}

\begin{document}

\maketitle
\vspace{-14mm}

\begin{abstract}
    Machine learning interatomic potentials (MLIPs) enable efficient and accurate atomistic simulations but depend critically on the quality and diversity of the training data. We introduce Stein kernelized molecular dynamics (SKMD), an enhanced sampling method that uses interacting particle dynamics to acquire informative training configurations for the fine-tuning of MLIPs. SKMD corresponds to a stochastic variant of Stein variational gradient descent that is adapted for molecular dynamics by incorporating asynchronous particle updates and a kernel of global atomic descriptors, which provides a symmetry-aware measure of configurational similarity. Unlike other enhanced samplers used in molecular dynamics, SKMD maintains the Boltzmann distribution as a stationary distribution of its dynamics in the appropriate asymptotic limit. This property enforces a balance between the exploration of diverse configurations and attraction toward high-probability regions of the energy landscape. We further propose an approach to efficient online data acquisition using an adaptive stopping criterion that selects non-redundant training data over the course of simulation. We demonstrate SKMD for the active learning of a neural network model of the M\"uller--Brown potential, a MACE interatomic potential for alanine dipeptide, and a MACE foundation model with a DFT oracle for predicting the Pandey reconstruction of the Si(111) surface. Compared to active learning baselines, our method achieves higher model accuracy in fewer training iterations.
\end{abstract}

\section{Introduction}
\label{sec:intro}

Many modern advances in the atomistic simulation of chemical phenomena stem from the use of machine learning interatomic potentials (MLIPs), data-driven surrogate models of atomistic force fields that enable molecular dynamics (MD) simulation at greater system sizes and time scales than what is feasible with \textit{ab initio} methods. The accuracy of MLIPs depends critically on the quality of the training data: training configurations must be representative of both key thermodynamic states of the system and the transitional states bridging between them in order for the MLIP to correctly characterize chemical properties. Training data of these transitional states, or of unobserved thermodynamic states, are challenging to acquire due to the infrequency of transitions during simulation. Moreover, the high cost of labeling data with quantum-mechanical reference calculations limits the number of samples which can be feasibly added to the training set.  

In \textit{active learning} of MLIPs, we progressively improve the accuracy of the model by alternating between the collection of training data and retraining of the model on the augmented dataset. A wide array of existing active learning approaches for MLIPs leverage subset selection methods---based, e.g., on D-optimal design \cite{Podryabinkin2017}, CUR decomposition \cite{Bernstein2019}, the MaxVol algorithm \cite{Lysogorskiy2021,Lysogorskiy2023}, Gaussian process variance \cite{Jinnouchi2019,Vandermause2020,Vandermause2021,Xie2021,Xie2022}, Bayesian selection~\cite{rocke_bayesian_2025}, entropy-maximization \cite{Karabin2020,Subramanyam2025}, query-by-committee \cite{Artrith2012}, and determinantal point processes \cite{Zou2025}---for selecting informative subsets of an MD trajectory simulated with the MLIP. The subset of configurations is then labeled with reference potential energy and force calculations and added to the training set. However, MD trajectories produced by standard techniques, including Langevin dynamics, are susceptible to remaining trapped in energy basins. Thus, these approaches may yield highly correlated data irrespective of the subset selection technique, which limits model improvement.

For this reason, recent active learning methods utilize \textit{enhanced sampling} in molecular dynamics to promote the exploration of novel regions of configuration space, including metadynamics \cite{Herr2018}, uncertainty-driven dynamics \cite{Kulichenko2023}, and hyperactive learning \cite{vanderOord2022}. These methods introduce an adaptive biasing force which drives dynamics toward underrepresented regions of the configuration space and define acquisition criteria for selecting nonredundant training data over the course of simulation. However, the biased dynamics do not retain fidelity to the Boltzmann distribution associated with the MLIP, and the acquisition criteria generally do not take into account the underlying energy landscape. Therefore, the chosen training configurations may not be representative of physically meaningful configurations or the true distribution of thermodynamic states. 

We address these problems with Stein kernelized molecular dynamics (SKMD), a novel enhanced sampling method for molecular dynamics adapted from variational inference approaches for Bayesian inference. SKMD is derived from Stein variational gradient descent (SVGD) \cite{Liu2016}, a particle-based variational inference algorithm which utilizes an interacting particle set to approximate a target distribution. Compared to existing enhanced samplers, our method offers several advantages: first, the Boltzmann distribution is a stationary distribution of the mean-field limit of the interacting particle dynamics, whereas most enhanced samplers target a modified distribution that may not be known analytically. Second, SKMD offers tunable control over the dynamical behavior of the algorithm, where the balance between fidelity-seeking and diversity-seeking dynamics can be adjusted directly through parameters of the algorithm in a manner consistent with theory. Finally, the SKMD biasing force offers a means to define an acquisition criterion which balances the selection of diverse configurations with those informed by the energy landscape.

Active learning, or the process of adaptively selecting informative training data, is often implemented for \textit{model fine-tuning}, where new data are used to update the parameters of a pretrained model and improve its accuracy in a target domain. While flow-based generative methods have been adopted for Boltzmann sampling and fine-tuning \cite{Noe2019,Plainer2025,Nam2025}, they require that training data already exist in regions which are targeted for additional sampling and can struggle to sufficiently sample regions with poor data coverage. We argue that an enhanced sampling framework is better suited for the task of active learning, as local particle transforms enable the discovery of thermodynamic states which are underrepresented in the existing training dataset. We summarize our contributions as follows: 
\begin{itemize}[leftmargin=*,nosep]
    \item We propose SKMD as a stochastic variant of SVGD implemented with asynchronous particle updates and a kernel of global atomic descriptors, thus adapting the algorithm to molecular dynamics settings.
    \item We prove that the asymptotic distribution of SKMD dynamics is characterized by a mean-field equation which has the Boltzmann distribution as a stationary solution in \Cref{prop:mfl}. The mean-field limit coincides with that of SVGD, which converges to the Boltzmann distribution under appropriate conditions.
    \item We develop an approach to online data acquisition in the form of an adaptive stopping criterion for SKMD, which automatically determines an appropriate query timing during simulation.
    \item We show that SKMD outperforms or matches other sampling techniques for data generation and active learning of MLIPs, demonstrated on problems of learning a neural network model of the M\"uller-Brown potential and fine-tuning MACE foundation models to describe alanine dipeptide and the silicon (111) surface. 
\end{itemize}

\vspace{-2mm}
\section{Background}
\label{sec:background}

\subsection{Machine learning atomistic force fields}

\textbf{Interatomic potentials} model the total potential energy of a configuration of $N$ atoms as a function of the atomic positions $\x = (x_1, \ldots, x_N) \in \mathbb{R}^{3N}$, where $x_n \in \mathbb{R}^3$. Whereas classical empirical potentials \cite{Tersoff1988,Stillinger1985,Daw1984,Baskes1994} are analytical functions taking simple parametric forms, machine learning interatomic potentials (MLIPs) are flexible function approximations learned from higher-fidelity reference data such as density functional theory (DFT) calculations. Let $V: \mathbb{R}^{3N} \to \mathbb{R}$ and $-\nabla_\x V: \mathbb{R}^{3N} \to \mathbb{R}^{3N}$ correspond to the reference potential energy and forces, referred to throughout as the \textit{oracle}. A MLIP $V_\theta: \mathbb{R}^{3N} \to \mathbb{R}$ is typically learned from a weighted least squares objective, where the model parameters $\theta \in \Theta$ are the minimizer of the loss function $\mathcal{L}$,
\begin{equation}
    \label{eq:loss}
    \mathcal{L}(\theta) = \frac{\lambda_0}{K} \sum_{k=1}^K |V_\theta(\x^k) - V(\x^k) |^2 + \frac{\lambda_1}{K} \sum_{k=1}^K || \nabla_\x V_\theta(\x^k) - \nabla_\x V(\x^k) ||^2 \ ,
\end{equation}
defined by $\lambda_0, \lambda_1 > 0$ and evaluated at a training set $\mathcal{D} \coloneqq \{ (\x^k, V(\x^k), \nabla_\x V(\x^k )) \}^K_{k=1}$.

\textbf{Atomic descriptors} are feature representations of local atomic environments which form the basis of many MLIPs. A local descriptor $\tilde{g}(\x) \in \tilde{\Omega} \subseteq \mathbb{R}^d$ is either learned from data, as with invariant latent representations from equivariant NN potentials such as NequIP \cite{Batzner2022}, Allegro \cite{Musaelian2023}, or MACE \cite{Batatia2022,Batatia2025}, or constructed explicitly from symmetry-adapted bases of local atomic neighborhoods which enforce invariances under SO(3), such as SOAP descriptors \cite{Bartok2010}, bispectrum components \cite{Thompson2015}, and ACE basis functions \cite{Drautz2019}. Whereas local descriptors are per-atom representations, \textit{global descriptors} $g(\x) \in \Omega \subseteq \mathbb{R}^d$ are representations of the multi-atom configuration, typically a composition of local descriptors $\tilde{g}(\x) = [\tilde{g}^1(\x),\ldots,\tilde{g}^N(\x)] \in \tilde{\Omega}^N \subseteq \mathbb{R}^{Nd}$. An example of a global descriptor which is used in \cite{Karabin2020} is the mean of the local descriptors, $g(\x) = \frac{1}{N} \sum_{n=1}^N\tilde{g}^n(\x)$. Depending on the application, a more informative global descriptor could be the mean of local descriptors of a subset of the atoms, e.g., of interfacial atoms within a bulk configuration.

\textbf{Boltzmann sampling} in molecular dynamics consists of generating samples distributed according to the Boltzmann distribution of configurations of the atomic system, $\pi(\x) = \frac{1}{Z} \exp \big( -\beta V(\x) \big)$, where $Z$ is the normalizing constant. Given a MLIP which is cheaper to evaluate than the oracle, one generally approximates $\pi$ with the Boltzmann distribution of the MLIP, $\pi_\theta(\x) = \frac{1}{Z_\theta} \exp \big( -\beta V_\theta(\x) \big)$. A common approach to Boltzmann sampling is to perform molecular dynamics simulation with a Langevin thermostat, which corresponds to underdamped Langevin dynamics. If the dynamics are ergodic, then the marginal invariant distribution of positions coincides with $\pi_\theta$ \cite{Lelievre2010}. 

In the overdamped limit of Langevin dynamics, $\pi_\theta$ remains the marginal invariant distribution of positions and the molecular dynamics according to
\begin{equation}
    \label{eq:langevin_od}
      \upd \x_t = -\nabla_\x V_\theta(\x_t) \upd t + \sqrt{2 \beta^{-1}} \upd W_t \
\end{equation}
can be used to sample from the Boltzmann distribution. However, since the Boltzmann distribution of molecular systems is typically highly multi-modal, trajectories of either underdamped or overdamped Langevin dynamics can remain confined for long times within metastable basins separated by high barriers in the free energy landscape. On practical simulation time scales, this metastability results in slow mixing and poor coverage of the full distribution \cite{Henin2022}.

\subsection{Particle-based variational inference}

\textbf{Stein variational gradient descent (SVGD)} is a particle-based variational inference algorithm which approximates a target density $\pi$ on a state space $\mathcal{X}$ with the empirical distribution $\hat{q}_t$ of a set of interacting particles $\bar{X}_t = \{ \x^i_t \}_{i=1}^J$. For time step $\epsilon > 0$ and symmetric positive semi-definite kernel $k: \mathcal{X} \times \mathcal{X} \to \mathbb{R}$, the evolution of each particle $i=1,\ldots, J$ is given by
\begin{subequations}      
        \begin{equation}
        \label{eq:steindiscrete}
        \x^i_{t+1} \leftarrow \x^i_{t} + \epsilon \hat{\phi}^*_t(\x^i_t; \bar{X}_t) \ ,
        \end{equation}
        \begin{equation}
        \label{eq:steinmc}
        \hat{\phi}^*_t(\cdot; \bar{X}_t) = \frac{1}{J} \sum_{j=1}^J \Big[ \nabla_{\x^j_t} \log \pi(\x^j_t) k(\x^j_t, \cdot) + \nabla_{\x^j_t} k(\x^j_t, \cdot) \Big], \quad \x^j_t \in \bar{X}_t \ .
        \end{equation}
\end{subequations}
One can show that \eqref{eq:steindiscrete} corresponds to the Euler discretization of a continuous-time ODE describing the particle dynamics, and \eqref{eq:steinmc} corresponds to a Monte Carlo estimator of an expectation resulting from Stein's identity \cite{Stein1972},
\begin{subequations}      
        \begin{equation}
        \label{eq:steinode}
        \upd\x^i_t = \phi^*_t(\x^i_t) \upd t \ ,
        \end{equation}
        \begin{equation}
        \label{eq:steinexpec}
        \phi^*_t(\cdot) = \mathbb{E}_{\x \sim \hat{q}_t} [ \nabla_{\x} \log \pi(\x) k(\x, \cdot) + \nabla_{x} k(\x, \cdot)] \ .
        \end{equation}
\end{subequations}
Under appropriate conditions, in the limit of infinite time $t \to \infty$ and infinite particles $J \to \infty$, the mean-field density $q_t$ converges weakly to $\pi$ in Kullback--Leibler (KL) divergence \cite{Liu2017,Lu2019}. App.~\ref{app:varinf} describes the SVGD formulation in further detail.

\section{Stein kernelized molecular dynamics}
\label{sec:methods}

In \Cref{sec:sampling}, we present SKMD as a sampling algorithm that may be used for approximation of the Boltzmann distribution. In \Cref{sec:acquisition}, we discuss approaches to adapt the sampling algorithm for active learning, where the goal is to select new training configurations that improve model quality while managing the cost of data acquisition and labeling.

\subsection{Sampling algorithm}
\label{sec:sampling}

Let $\bar{X}_s \coloneqq \{ \x_s^{j} \}_{j=1}^J$ be an ensemble of $J$ particles at a given time step $s > 0$ and let $\bar{X}^{-i}_s \coloneqq \bar{X}_s \setminus \{\x_s^{i}\}$ be the ensemble excluding the $i$th particle. Assume $k: \mathbb{R}^n \times \mathbb{R}^n \to \mathbb{R}$ is a symmetric positive semi-definite kernel. For stopping time $\ell > 0$, step size $\epsilon > 0$, constant $\eta > 0$, scale parameter $A: \mathbb{R}^n \to \mathbb{R}$, and $\xi^{i}_t \sim N(0,\mathbb{I}_n)$, the evolution of the $i$th particle in the set for $t = s, ..., s+\ell-1$ is given by
    \begin{subequations}    
    \label{eq:skmd}
        \begin{equation}
        \label{eq:skmd_update}
        \upd\x^{i}_{t+1} \leftarrow \x^{i}_{t} + \epsilon \big[- A(\x^{i}_t) \beta \, \nabla_{\x^{i}_t} V_\theta(\x^{i}_t) + F^\textup{SKMD}_{\theta,s}(\x^{i}_t; \bar{X}^{-i}_s) \big] + \sqrt{\tfrac{2 \epsilon \eta}{J}} \ \xi_t^{i} \ ,
        \end{equation}
        \begin{equation}
        \label{eq:skmd_force}
        F^\textup{SKMD}_{\theta,s}(\cdot; \bar{X}_s) = \frac{1}{J-1} \sum_{j=1}^{J-1} \Big[ - \beta \,
        \nabla_{\x_s^{j}} V_\theta(\x_s^{j}) k(\x_s^{j}, \cdot) + \nabla_{\x_s^{j}} k(\x_s^{j}, \cdot) \Big], \ \x_s^{j} \in \bar{X}^{-i}_s .
        \end{equation}
\end{subequations}
The ensemble $\bar{X}^{-i}_s$ remains fixed until after $\ell$ steps, at which point the current particle $\x^i_{s+\ell}$ replaces $\x^i_s$ in the ensemble and we switch to evolving the next particle in the ensemble at step $s+\ell$. The sampling algorithm, summarized in \Cref{alg:samp}, has the following characteristics:

\textbf{Adaptive biasing force.} The force field in \eqref{eq:skmd_update} is the combination of $- \beta \, \nabla_{\x^{i}_t} V_\theta(\x^{i}_t)$, referred to as the \textit{gradient force}, and $F^\textup{SKMD}_{\theta,s}(\x^{i}_t; \bar{X}_s)$, referred to as the \textit{SKMD biasing force}. In the first term in the summand of \eqref{eq:skmd_force}, the gradient force at each particle in the ensemble scaled by the kernel acts as an attractive force that draws the trajectory toward low-energy configurations to promote fidelity to the free energy landscape. The second term, the gradient of the kernel, acts as a repulsive force that drives particles apart in order to promote exploration of the configuration space. The attractive and repulsive forces strike a balance between the \textit{exploration} of novel configurations and \textit{exploitation} of high-probability regions which promotes fidelity to the asymptotic distribution of samples. 

We introduce a scale parameter $A$ to differentially weigh the strength of the gradient force relative to the SKMD biasing force at the current point in the simulation path. By setting $A(\x) = k(\x,\x) = a$, $\forall \x \in \mathbb{R}^n$, the drift of \eqref{eq:skmd_update} coincides with the velocity field of SVGD at $t=s$. Setting a larger kernel amplitude $a > 0$ can enhance the repulsive effect of the dynamics, but setting $A=a$ would in turn magnify the scale of the negative gradient of the potential, making it more difficult for the path to cross an energy barrier. Therefore, for greater flexibility, we do not strictly require that $A$ equal the kernel self-similarity in the SKMD algorithm. In our implementation, we generally set $A = 1$. 

\textbf{Kernel of global descriptors}. The kernel measures the degree of similarity between two configurations. We assume that configurations are more clearly separated in the Euclidean geometry of the descriptor space $\Omega$ than in Cartesian space, since the descriptor map encodes invariance or equivariance relations. Therefore, we use translation-invariant kernels based on distances between the descriptors; for example, a Gaussian kernel with amplitude $a > 0$ and length scale $b > 0$,
\begin{equation}
    k(\x,\x') = k_g(g(\x), g(\x')) = a \exp \Big(-\frac{||g(\x) - g(\x')||^2}{2b^2} \Big) \ .
\end{equation}
The gradient of the kernel with respect to its first argument is computed via chain rule as
\begin{equation}
    \nabla_\x k(\x,\x') = -\frac{1}{b^2} k(\x,\x') \mathcal{J}(\x)^\top \big( g(\x) - g(\x') \big) \ ,
\end{equation}
where the Jacobian $\mathcal{J} \in \mathbb{R}^{d \times 3N}$ is a block matrix with $\mathcal{J}_{ij}(\x) = \frac{\partial g_i}{\partial \x_j}(\x)$. The works of \cite{Grisafi2018,Huo2022} further address the development of characteristic symmetry-aware kernels to use in kernel-based machine learning methods in molecular dynamics.

\textbf{Asynchronous particle updates.} In typical interacting particle systems, the positions of all particles in the ensemble are updated at every time step, which we refer to as the \textit{synchronous} update rule. We propose an \textit{asynchronous} update rule where particles are advanced one at a time, each for a finite number of steps $\ell$. The asynchronous update rule allows one to use single-particle simulation to approximate the ensemble dynamics of the interacting particle system. This approach both reduces the computational overhead associated with inter-particle communication and simplifies the integration of SKMD into existing molecular dynamics workflows, such as those in LAMMPS \cite{Thompson2022}. We evaluate the effect of asynchronous particle updates on the fidelity of SKMD samples to the Boltzmann distribution in App.~\ref{app:samp}. 

\textbf{Isotropic stochastic noise.} The addition of stochastic noise in \eqref{eq:skmd_update} serves to improve the mixing of the simulation path for small ensemble sizes ($O(10)$) and relates the update equation to a stochastic variant of SVGD first proposed in \cite{Gallego2018}. In \cite{Gallego2018,Duncan2023}, stochastic SVGD is defined by a kernel-dependent diffusion coefficient. However, we found numerically that the asynchronous update rule in \eqref{eq:skmd_update} with the isotropic, state-independent diffusion coefficient produces samples that exhibit better agreement with the Boltzmann distribution compared to the implementation with the kernel-dependent diffusion coefficient. A state-independent diffusion coefficient is also adopted in \cite{Ye2020} for a single-trajectory variant of stochastic SVGD. Our choice of diffusion coefficient does not perturb the mean-field limit of the dynamics, as shown in App.~\ref{app:asymptotic}. 

\textbf{Consistent mean-field dynamics.}  \Cref{prop:mfl} shows that in the limit $\epsilon \to 0$, $\ell \to 0$, and $J \to \infty$ and under appropriate assumptions, the empirical distribution of the particles, $\hat{q}^J_t(\x) \coloneqq \frac{1}{J} \sum_{j=1}^J \delta(\x - \x_t^{j})$, converges to the solution to a mean-field equation, $q_t$. This equation coincides with the mean field limit of SVGD and admits the Boltzmann distribution $\pi_\theta$ as a stationary solution. Consequently, SKMD inherits properties of SVGD that are derived from the mean-field equation, including existing results on conditions under which $q_t$ converges weakly to $\pi_\theta$ as $t \to \infty$ \cite{Lu2019,Duncan2023}. Moreover, since the mean-field dynamics are a gradient flow that decreases KL divergence with respect to $\pi_\theta$, one can interpret the local attractive and repulsive dynamics in configuration space as corresponding to global transport of the empirical measure of particles in probability space. See App.~\ref{app:asymptotic} for the necessary assumptions and further discussion. 

The SKMD algorithm is constructed as a consistent discrete approximation of the idealized equation in \eqref{eq:skmd_cts} for which we have theoretical guarantees. In particular, one can coarsen or refine the discretization of SKMD to achieve a balance between computational efficiency and accuracy with respect to the synchronous, continuous-time, infinite-particle formulation. This feature of our algorithm distinguishes it from other enhanced samplers such as query-by-committee methods, for which theoretical limits have not been proven. While theoretical characterization of the asynchronous, finite-particle algorithm is beyond the scope of this paper, we empirically evaluate the influence of the time lag $\ell$, particle set size $J$, and kernel parameters on the performance of the sampling algorithm in the ablation studies of App.~\ref{app:ablation}.

\begin{algorithm}[t]
\caption{SKMD for active learning with online data acquisition}
\label{alg:al_online}
\begin{algorithmic}
    \Require{Initial ensemble $\bar{X}_0 = \{\x^{1}_0, \ldots, \x^{J}_0\}$, 
            training set $\mathcal{D}_0$, model $V_{\theta_0}$, kernel $k$, threshold $\zeta_0$} 
    
    \For{training iteration $\tau = 0,1,\ldots$}
        \For{ensemble member $i = 1,\ldots,J$}
            \State Initialize trajectory at $\x_0 \gets \x^{i}_\tau$ 
            \Repeat{ for $t=0,1,\ldots$}
                \State Simulate one step of \eqref{eq:skmd}: $\x_{t+1} \gets \text{SKMD}(\x_t, \bar{X}^{-i}_\tau, V_{\theta_\tau}, k)$
            \Until{acquisition criterion \eqref{eq:stopping} is met at $\x_{t+1}$ with threshold $\zeta_0$}
            \State Add $\x_{t+1}$ to the set $\mathcal{D}_{\tau+1}$
            \State Update ensemble: replace $\x^{i}_\tau$ in $\bar{X}_{\tau}$ with $\x_{t+1}$
        \EndFor
        \State Train new model $V_{\theta_{\tau+1}}$ on the augmented training set $\bigcup_{k=0}^{\tau+1} \mathcal{D}_{k}$
    \EndFor
    \State \Return Final model $V_{\theta_\text{final}}$ and training set $\mathcal{D}_\text{final}$
\end{algorithmic}
\end{algorithm}

\subsection{Data acquisition}
\label{sec:acquisition}

To adapt it into an active learning algorithm, SKMD may be paired with either an offline or online data acquisition step. We review such methods below.

\textbf{Offline data acquisition} is performed by choosing a subset of a pre-collected pool of candidate configurations to label with the oracle and add to the training data. In many such methods, the chosen subset is the one that maximizes some functional measuring the dissimilarity of elements---such as the determinant of a feature kernel, as with MaxVol algorithm \cite{Lysogorskiy2021, Lysogorskiy2023} and determinantal point processes \cite{Derezinski2020, Zou2025}, or the distance between elements, as with farthest point sampling (FPS). These methods may be directly adopted by utilizing samples generated by SKMD as the pool of candidate configurations, as described in \Cref{alg:al_offline}. The influence of SKMD parameters on its performance with offline data acquisition is evaluated in the ablation studies of App.~\ref{app:ablation}. Offline methods for data acquisition often incur high computational costs, as the selection of subsets from a candidate pool is a combinatorially hard problem. Practical algorithms rely on greedy approximations of optimal subset selection, which generally requires repeated matrix determinant evaluations whose cost scales with the total set size. 

\textbf{Online data acquisition} is a cost-effective alternative to offline approaches, performed by specifying criteria for collecting training data during simulation. While the sampling algorithm in \eqref{eq:skmd} is defined for a fixed stopping time $\ell > 0$, SKMD can be oriented into a technique for online data acquisition by introducing an adaptive stopping criterion. Let the Euclidean norm of the SKMD biasing force define an acquisition function $\alpha_{s}$. For a given particle that evolves from a time $s \geq 0$, the adaptive stopping time $\ell(s;\zeta_0)$ is the first time after a set number of steps $\ell_0$ that acquisition function falls below a threshold $\zeta_0 > 0$, i.e.,
\begin{equation}
    \label{eq:stopping}
    \ell(s;\zeta_0) \coloneqq \inf \{t \geq \ell_0 : \alpha_{s}(\x_{s+t}) < \zeta_0 \}, \qquad \alpha_{s}(\x) = || F^\textup{SKMD}_{\theta,s}(\x; \bar{X}_s) || \ .
\end{equation}
At the point $\x_{s+\ell}$ where the criterion is met, we collect that point as a new training datum and switch the simulation to the next particle in $\bar{X}_s$. After advancing each particle in the ensemble and collecting $J$ configurations, we label the data with reference calculations and retrain the model on the augmented training set. This approach is outlined in \Cref{alg:al_online}. 

The acquisition criterion is a Stein-informed heuristic for when a given point is informative with respect to both the potential energy surface and the existing training set. It promotes the selection of points at key geometric features of the potential energy surface, since the gradient of the potential in the SKMD biasing force is small in energy basins and at saddle points. Moreover, it promotes the selection of points which are distinct from previously added training data, since the kernel and its gradient become small when the point is far from $\bar{X}_s$ in descriptor space. 

\section{Related work}
\label{sec:relatedwork}

\textbf{Enhanced sampling} \cite{Henin2022} encompasses a broad class of methods used to promote some desired sampling behavior in molecular dynamics, reviewed in App.~\ref{app:enhancedsampling}. The method most closely related to our application is uncertainty-driven dynamics (UDD) \cite{Kulichenko2023},
\begin{equation}
    \label{eq:udd}
      \upd \x_t =  \big( - \nabla_\x V_\theta(\x_t) - \nabla_\x V^{\textup{UDD}}_t(\x_t) \big) \upd t + \sqrt{2\beta^{-1}} \upd W_t \ ,
\end{equation}
which introduces a biasing potential based on a Gaussian kernel of the empirical variance of energy predictions by a committee of MLIPs. The committee is composed of the MLIP with different parameters learned from $k$-fold splits of the training data. Online data acquisition is performed by querying points at which the committee uncertainty exceeds a certain threshold. Hyperactive learning (HAL)~\cite{vanderOord2022} employs a sampler of the same form, but with the committee parameters drawn from a Bayesian posterior. Both methods are \textit{query-by-committee} algorithms designed to preferentially sample regions of high uncertainty, but unlike SKMD, they do not maintain the Boltzmann distribution as the invariant distribution and do not incorporate information on the free energy landscape in the acquisition criteria. Moreover, because UDD requires the training and simulation of multiple MLIPs, each active learning iteration with UDD incurs a higher computational cost than with SKMD, as reflected by their run times and memory usage detailed in App.~\ref{app:compute}. 

\textbf{Kernel herding} constructs a set of representative points whose empirical distribution closely matches a target distribution $\pi$ in terms of their kernel mean embeddings. The SKMD acquisition criterion can be compared to kernel herding via Stein points \cite{Chen2018}, where points are selected greedily according to 
\begin{equation}
    \label{eq:steinpts}
    \x_n \in \underset{
    \x \in \mathcal{X}}{\text{argmin}} \sum_{i=1}^{n-1} \mathcal{S}^{\x_i}_{\pi} \mathcal{S}^{\x}_{\pi} \otimes k(\x_i, \x)
\end{equation}
given previously chosen points $\{ \x_1,...,\x_{n-1}\}$ and Stein operator $\mathcal{S}^\x_\pi$. In our acquisition criterion, the previously chosen points are training data $\bar{X} \subseteq \mathcal{D}$, the target distribution is the Boltzmann distribution $\pi_\theta$, and rather than searching for a minimum over the full configuration space $\mathcal{X}$, we select points along the simulation path with low values of $\sum_{i=1}^{J-1} \mathcal{S}^{\cdot}_{\pi_\theta} \otimes k(\x_i, \cdot)$ in terms of its Euclidean norm. Therefore, the SKMD acquisition scheme can be interpreted as an approach to select points which reduce a simpler proxy objective to the one in Stein points. We discuss this relation in App.~\ref{app:herding}.

\textbf{SVGD} has spawned a plethora of variants of Stein repulsive dynamics, for applications in Bayesian inference \cite{Gallego2018,Detommaso2018,Wang2019,Li2020,Ye2020,Alsup2021}, data assimilation \cite{Stordal2021}, and POMDPs \cite{Zhang2025}. To our knowledge, our work is the first to derive an explicit active learning framework from principles of SVGD.
\section{Experiments}
\label{sec:experiments}

\subsection{Neural network potential}
\label{sec:nn}

We first evaluate the performance of SKMD for active learning (AL) of a neural network (NN) potential. As the oracle, we use the Müller-Brown potential (\Cref{fig:al_grid}a), a common benchmark with three metastable basins for studying problems in molecular dynamics. The initial training set is concentrated in a single basin (\Cref{fig:al_grid}b) to emulate conditions where existing configurations are highly correlated. We compare four samplers, each acquiring 32 new data points per AL iteration: 
\begin{itemize}[leftmargin=*,nosep]
    \item overdamped Langevin dynamics, with data selected at fixed intervals along the trajectory;
    \item UDD \cite{Kulichenko2023}, with data selected online using the uncertainty criterion \eqref{eq:uncertainty};
    \item SKMD, with data selected at fixed intervals along the trajectory;
    \item a-SKMD, with data selected online using the adaptive criterion \eqref{eq:stopping}.
\end{itemize}

For fair comparison, all schemes share the same initial model (\Cref{fig:al_grid}b), training loss \eqref{eq:loss}, temperature, and realization of Brownian motion in each AL trial, such that differences in their performance only reflect the choice of sampler and acquisition criterion. Further experimental details are in App.~\ref{app:exp_nn}.

\Cref{fig:al_grid}c--e illustrates how each sampler explores the state space over AL iterations. At the given temperature, samples from Langevin dynamics remain confined to the initial basin and the model is not able to resolve the other two basins. The samples obtained by UDD tend to cluster in close proximity to each other. This behavior is due to the fact that the uncertainty criterion may be large only in concentrated regions, as discussed in App.~\ref{app:adaptstop}. Trajectories of a-SKMD exhibit faster mixing and quickly discover both unseen basins. Moreover, the queried points exhibit more even spread in the state space, indicating that the SKMD stopping criterion is effective in promoting dissimilarity in tis queries.

These differences are reflected in the model accuracy, measured in terms of the root mean squared error (RMSE) in energy and force RMSE over 50 trials on a test set of samples from the reference Boltzmann distribution (\Cref{fig:al_grid}f--g). Among the fixed-interval schemes, SKMD leads to lower error than the Langevin scheme. Among the adaptive schemes, a-SKMD more rapidly decreases error compared to UDD. a-SKMD shows significant improvements over the other methods, converging to a lower value of error with smaller variance in the fewest iterations.

    \begin{figure}
    \centering
    \includegraphics[width=\textwidth]{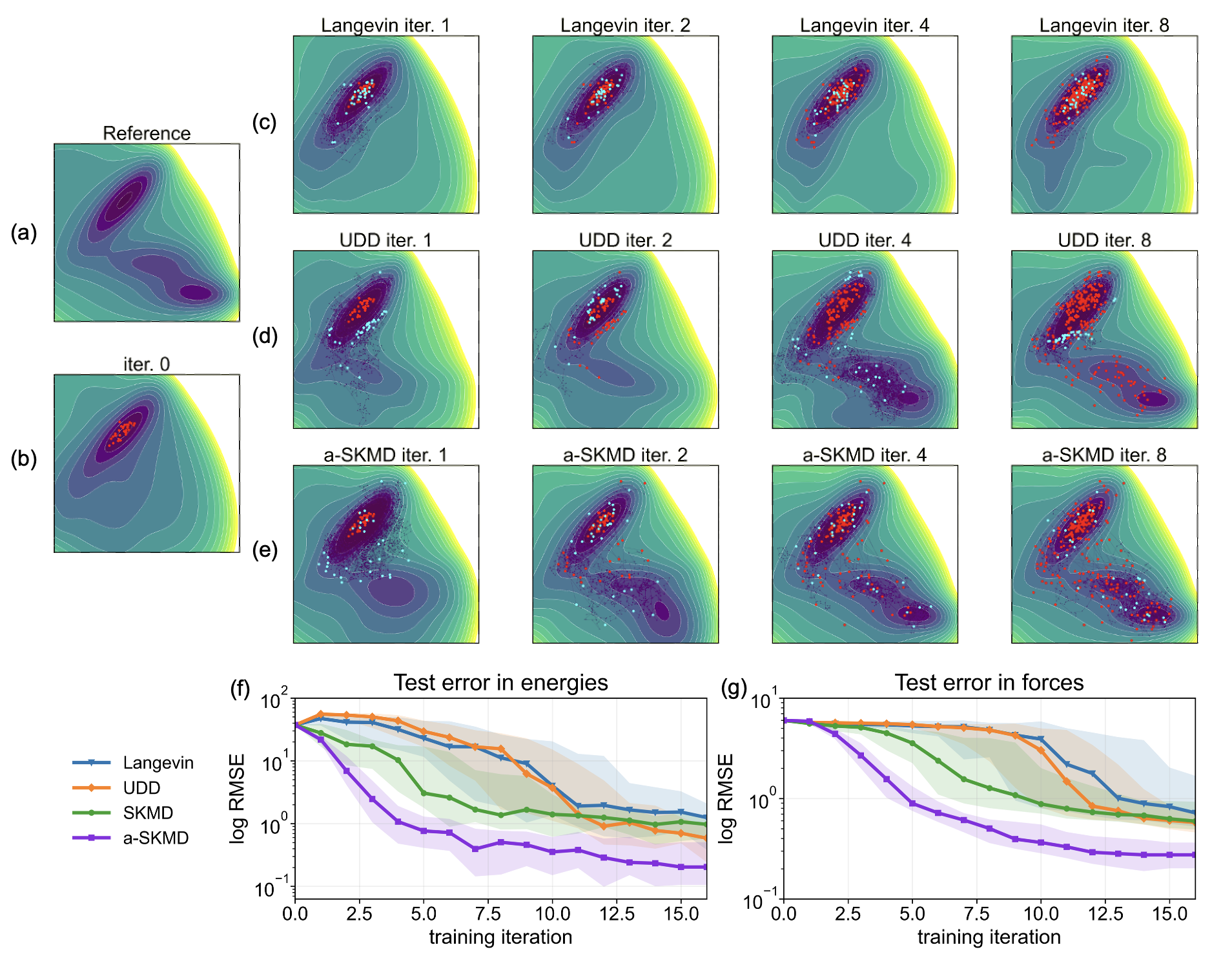}
    \caption{\small{Active learning of an NN potential, initialized at (b), using the M\"uller--Brown potential as the oracle (a). Contours of the NN potential at iterations $\tau \in\{ 1,2,4,8 \}$ of active learning by overdamped Langevin dynamics (c), UDD (d), and a-SKMD (e). The accumulated training data are shown in red, queried data at the current iteration in cyan, and path from the previous stopping time to the current stopping time in dark blue. a-SKMD leads to the fastest reduction in root mean squared error (RMSE) in the potential energy (f) and forces (g) of the NN potential, where 32 additional data are used in each AL iteration. Error is reported in terms of the median and 25th to 75th percentile range across 50 trials.}}
    \label{fig:al_grid}
    \end{figure}

\subsection{Alanine dipeptide}
\label{sec:mace}

We next use SKMD to fine-tune a MACE potential for the alanine dipeptide molecule, a standard test case for active learning and enhanced sampling~\cite{Schwalbe2021, Zaverkin2024} whose Ramachandran $(\psi,\phi)$ energy landscape contains multiple well-separated minima (Figure~\ref{fig:mace_fig_1}a). So that labels are cheap and the ground truth is known exactly, we take MACE-OFF-23-small~\cite{Kovacs2025} as the oracle and a simpler MACE model as the initial surrogate. The global descriptor is the sum of the atomic MACE invariant features from the first two layers ($d=64$). As alanine dipeptide has $N=22$ atoms, the dynamics evolve in the configuration space $\mathbb{R}^{66}$. We compare six samplers, each acquiring up to 100 new data points per AL iteration:
\begin{itemize}[leftmargin=*, nosep]
  \item underdamped Langevin MD at 300\,K and 700\,K, with data selected by FPS;
  \item overdamped Langevin MD at 50\,K, with data selected by FPS;
  \item SKMD running overdamped dynamics at 50\,K, with data selected by FPS;
  \item uHAL~\cite{vanderOord2022, rocke2025aseuhal} underdamped at 700\,K, with data selected by the uncertainty criterion;
  \item a-SKMD overdamped at 50\,K, with data selected by the adaptive criterion.
\end{itemize}

 The temperature of the overdamped Langevin baseline matches that of SKMD so that it provides a comparison of the dynamics without the biasing force. We include underdamped Langevin dynamics as a standard baseline for MD simulation, where 300 K is room temperature (representative of typical conditions) and 700 K is an elevated temperature used to accelerate exploration of the configuration space. uHAL, short for ASE-uHAL~\cite{rocke2025aseuhal,rocke_bayesian_2025}, is included as a point of comparison to a state-of-the-art query-by-committee approach used in MD simulation for active learning. It is a variant of HAL \cite{vanderOord2022} built atop the Atomic Simulation Environment (ASE) \cite{Larsen2017} which is compatible with MACE potentials. The committee members are linear models on the fixed invariant MACE descriptor of a single baseline potential, with weights sampled from the posterior of a Bayesian linear regression. As this posterior covariance depends only on the descriptors and their gradients, the committee can be updated after each selected configuration without new oracle calls, allowing multiple configurations to be acquired per active learning iteration and labeled in a single batch. All samplers share the same simulation budget per iteration. Further experimental details are in App.~\ref{app:exp_mace}.

\begin{figure}[h!]
    \centering
    \includegraphics[width=\linewidth]{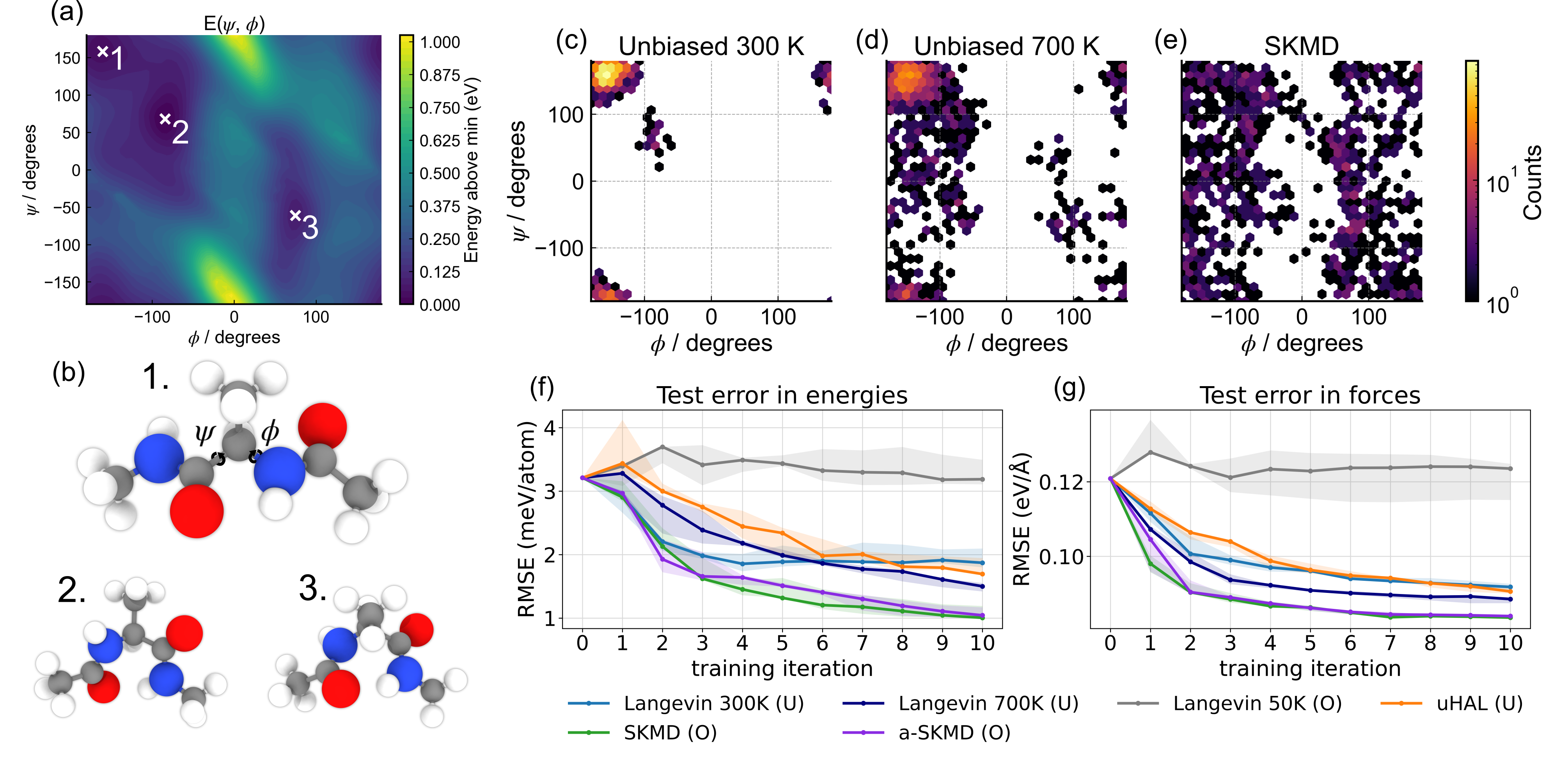}
    \caption{\small{Alanine dipeptide: SKMD explores the Ramachandran surface more broadly than unbiased MD and gives lower test errors than every baseline. (a) MACE-OFF-23-small Ramachandran surface $E(\phi,\psi)$. (b) Three minimum-energy conformations; $(\phi,\psi)$ are labeled on 1. (c)--(e) All 1000 samples from 10 AL iterations of unbiased MD at (c) 300~K and (d) 700~K, and (e) SKMD, overlaid on the surface in (a). (U) and (O) denote underdamped and overdamped dynamics. (f), (g) Test-set energy and force RMSE over 10 iterations for all six samplers (median and interquartile range over 5 trials).}}
    \label{fig:mace_fig_1}
\end{figure}

Figure~\ref{fig:mace_fig_1}c--e shows the Ramachandran coverage of all 1000 samples from representative runs. 700~K MD covers more of the surface than 300~K MD, but SKMD covers substantially more than either. Figure~\ref{fig:mace_fig_1}f--g shows the energy and force RMSE on a held-out test set of 100 structures obtained by constrained relaxation on a $10\times10$ $(\phi,\psi)$ raster, as medians and interquartile ranges over 5 independent trials. Over 10 AL iterations, errors decrease for all samplers apart from the overdamped 50~K baseline, showing that without the biasing force, dynamics in this regime are only weakly explorative. Within the shared simulation budget, uHAL's selection criterion triggered less often than the 100-configuration target, so it acquired only 35--48 configurations per iteration (about 400 in total, against 1000 for the other samplers). SKMD and a-SKMD lead to the fastest and largest reductions.  At a matched number of labels, SKMD remains more accurate: after 4 iterations (400 configurations), SKMD reaches 1.45~meV/atom and 0.0866~eV/\AA{}. a-SKMD performs comparably to SKMD (1.04~meV/atom, 0.0840~eV/\AA{}) while using on average 225 of the 300 outer-loop steps (61\% of the MD steps of the offline methods). For this system, online acquisition by adaptive stopping therefore selects data as informative as offline FPS over the full trajectory at a lower simulation cost.

\subsection{Silicon surface reconstruction}
\label{sec:silicon}

\begin{figure}
    \centering
    \includegraphics[width=\linewidth]{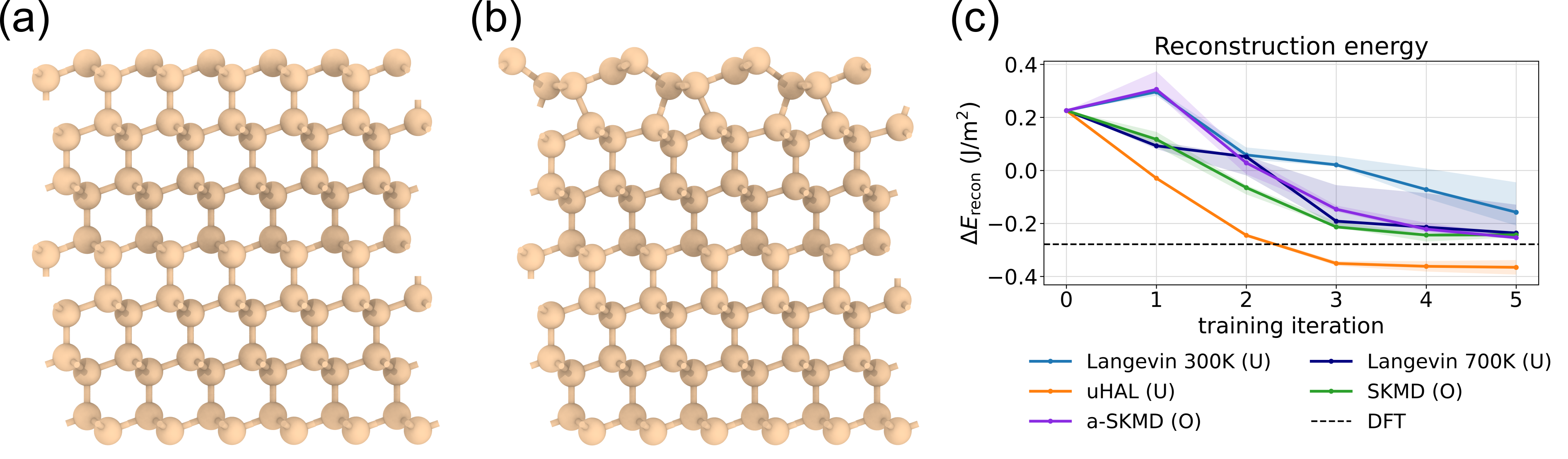}
    \caption{Using active learning to correctly predict the formation energy of the Pandey (2 $\times$ 1) reconstruction of the Si (111) surface. (a) The unreconstructed (111) surface. (b) The same domain with the top surface reconstructed. Note that for illustrative purposes three trials of this reconstruction are shown along the surface direction; the actual simulation domain was the same depth of slab, but had just one trial. (c) MACE-predicted reconstruction energy after each AL iteration, with both slabs relaxed with each model (median and interquartile range over 3 trials; 2 for uHAL). Methods are unbiased Langevin MD at 300~K and 700~K, uHAL, offline SKMD and a-SKMD; (U) and (O) denote underdamped and overdamped dynamics. The dashed line is the DFT value of $-0.279$~$\mathrm{J\,m^{-2}}$. uHAL reaches the correct sign first but overshoots, while SKMD, a-SKMD and 700~K Langevin converge closest to DFT.}
    \label{fig:si}
\end{figure}

To demonstrate SKMD on a condensed-phase system with a quantum-mechanical (DFT) oracle, we fine-tune a MACE foundation model to predict the Pandey $(2\times1)$ reconstruction of the Si(111) surface. Surface reconstructions involve significant electronic rehybridization and are a known failure mode of foundation models, which are fit mostly to bulk data. For Si(111), the natural cleavage plane, a correct description of this reconstruction is important to predict the correct fracture behaviour~\cite{kermode_low-speed_2008}. We model a 28-atom periodic $(2\times1)$ slab ($7.73\times3.87$~\AA{}) with a vacuum gap, in which the lowest 18 atoms are held fixed. The descriptor, constructed as in Section~\ref{sec:mace}, has dimension $d=256$ and the configuration space is in $\mathbb{R}^{84}$. We compare five samplers, each acquiring up to 30 new data points per AL iteration: 
\begin{itemize}[leftmargin=*, nosep]
  \item underdamped Langevin MD at 300\,K and 700\,K, with data selected by FPS;
  \item SKMD overdamped at 50\,K, with data selected by FPS;
  \item uHAL~\cite{vanderOord2022, rocke2025aseuhal} underdamped at 700\,K, with data selected by the uncertainty criterion;
  \item a-SKMD overdamped at 50\,K, with data selected by the adaptive criterion.
\end{itemize}

Further experimental details are in App.~\ref{app:exp_si}. Figure~\ref{fig:si} shows the reconstruction energy error, $\Delta E_{\mathrm{recon}}$ \eqref{eq:reconstruction}, over 5 AL iterations in terms of the median and interquartile range computed with 3 trials. Only two trials were completed for uHAL, as one uHAL trial acquired no configurations in the first iteration. uHAL predicts a stable reconstruction earliest, from the first iteration, but then overshoots the DFT reference value, lying 0.07--0.09~J\,m$^{-2}$ below it from iteration 3 onwards. SKMD converges fastest towards the DFT value, reaching within 0.07~J\,m$^{-2}$ after 3 iterations and 0.04~J\,m$^{-2}$ after 4, with the narrowest spread across trials. After 3 iterations, Langevin MD at 700 and 300~K lies 0.09 and 0.30~J\,m$^{-2}$ above DFT. a-SKMD's prediction worsens in the first iteration but reaches within 0.03~J\,m$^{-2}$ of DFT by iteration 5, using on average 130 of the 200 outer-loop steps (58\% of the MD steps of the offline methods). uHAL requires most of the budget in the first iteration (162 and 165 outer-loop steps), but in later iterations acquires its 30 configurations within 7--64 steps. After 5 iterations, a-SKMD (0.025~J\,m$^{-2}$), SKMD (0.036) and 700~K MD (0.042) lie closest to DFT, uHAL is 0.087~J\,m$^{-2}$ too low, and 300~K MD remains 0.12~J\,m$^{-2}$ too high. Overall, over 5 AL iterations, SKMD and a-SKMD converged closer to the DFT reconstruction energy than low-temperature unbiased MD and uHAL.
\section{Conclusions}
\label{sec:conclusions}

We introduce a novel framework to improve the efficiency of active learning and fine-tuning of MLIPs. Our method, SKMD, is defined by an adaptive biasing force that balances repulsive dynamics, which promote the exploration of new configurations, with attractive dynamics, which promote fidelity to the Boltzmann distribution.
In our numerical experiments with learning neural network surrogates of the Müller-Brown potential and fine-tuning MACE interatomic potentials for alanine dipeptide and Si(111) surface reconstruction, we find that SKMD has improved or comparable performance to existing baselines, reducing model error in the fewest active learning iterations. 

SKMD may be used as a general purpose algorithm for sampling configuration space as well as for online data acquisition. A limitation of the adaptive stopping criterion is that it requires SKMD to be defined by a kernel with a fixed bandwidth. When SKMD is defined by a variable kernel bandwidth (e.g., the median distance between particles), the norm of the SKMD biasing force defining the acquisition function does not decrease significantly over simulation time steps. In the MACE examples we overcome this with a two-phase scheme: an exploration phase, which uses SKMD with a variable kernel bandwidth and no acquisition to promote faster mixing in high-dimensional configuration space, followed by an exploitation phase, in which the interacting particles sample energy basins using SKMD with a fixed kernel bandwidth and adaptive stopping. On alanine dipeptide, this matches offline SKMD in accuracy while using fewer sampler steps. On Si(111), it also uses fewer steps and ends closest to DFT of all samplers, but improves more slowly in the first iterations, suggesting that thresholds calibrated during a short exploration phase transfer less reliably to higher-dimensional descriptor spaces. Alternative constructions of adaptive stopping criteria and a systematic choice of the kernel amplitude for high-dimensional systems are interesting directions for future work.

\begin{ack}
The work of JZ and YM was supported by the United States Department of Energy, National Nuclear Security Administration under Award Number DE-NA0003965. The work of FB was supported by a studentship from the UK Engineering and Physical Sciences Research Council–funded Centre for Doctoral Training in Modelling of Heterogeneous Systems (Grant No. EP/S022848/1). JZ and FB gratefully acknowledge the support of a Research Workshop Follow-on Grant from the International Centre for Mathematical Sciences, Edinburgh. FB further acknowledges the University of Warwick Scientific Computing Research Technology Platform for computational support, as well as the use of resources provided by the Isambard-AI National AI Research Resource (AIRR). Isambard-AI is operated by the University of Bristol and is funded by the UK Government's Department for Science, Innovation and Technology (DSIT) via UK Research and Innovation; and the Science and Technology Facilities Council [ST/AIRR/I-A-I/1023] \cite{isambardai}. Computations were also partly performed on the HPC system Raven at the Max Planck Computing and Data Facility (MPCDF).
\end{ack}

\printbibliography


\newpage 
\appendix

\section{Assumptions, Analysis, and Proofs}
\label{app:analysis}

\subsection{Asymptotic analysis of SKMD}
\label{app:asymptotic}

We provide a formal proof that in the appropriate asymptotic limits, the Boltzmann distribution is a stationary solution to the mean-field limit of SKMD. Therefore, the asymptotic behavior of \eqref{eq:skmd} is informed by the Boltzmann distribution, $\pi_\theta$, such that it can be implemented for Boltzmann sampling. First, we state the necessary assumptions, following the analysis in \cite{Duncan2023}. 

\begin{condition}
    \label{cond:kernel}
    The kernel $k: \mathcal{X} \times \mathcal{X} \to \mathbb{R}$ is continuous and symmetric positive semi-definite, i.e., it satisfies $\sum_{i,j=1}^J \alpha_i \alpha_j k(\x^{i}, \x^{j}) \geq 0$ for all $J \in \mathbb{N}$, $\alpha_1,...,\alpha_J \in \mathbb{R}$, and $\x^1, ..., \x^{J} \in \mathcal{X}$. 
\end{condition}

\begin{condition}
    \label{cond:target}
   The probability measure associated with the Boltzmann distribution $\pi_\theta$ belongs to $\mathcal{P}_k(\mathcal{X})$, a subset of $\mathcal{P}(\mathcal{X})$ defined as
    \begin{equation}
        \begin{aligned}
        \mathcal{P}_k(\mathcal{X}) \coloneqq \Big\{ \rho \in \mathcal{P}(\mathcal{X}) & : \rho \textup{ admits a smooth Lebesgue density}, \textup{supp } \rho = \mathcal{X}, \\
        & \int_{\mathcal{X}}k(\x,\x) \upd \rho(\x) < \infty \Big\} \ .
        \end{aligned}
    \end{equation}
\end{condition}

In the limit of $\epsilon \to 0$ and $\ell \to 0$, the SKMD update equation in \eqref{eq:skmd} converges to the stochastic differential equation
\begin{equation}
    \label{eq:skmd_cts}
    \upd\x^{i}_t = \frac{1}{J} \sum_{j=1}^J \Big[ -k(\x^{j}_t, \x^{i}_t) \beta \, \nabla_{\x^{j}_t} V_\theta(\x^{j}_t) + \nabla_{\x^{j}_t} k(\x^{j}_t, \x^{i}_t) \Big] \upd t + \sqrt{\tfrac{2 \eta}{J}} \upd W_t^{i}, \quad \x^{i}_t,\x^{j}_t \in \bar{X}_t \ ,
\end{equation}
for $A(\x) = k(\x, \x)$, where $W^{i}_t$, $i=1,...,J,$ are independent copies of $n$-dimensional standard Brownian motion. In \Cref{prop:mfl}, we establish that in the limit of $J \to \infty$, \eqref{eq:skmd_cts} has a mean-field equation that is identical in form to that of SVGD. The proof is provided in \Cref{app:proof}. 

\begin{proposition}[Mean field limit of SKMD]
\label{prop:mfl}
Assume \Cref{cond:kernel,cond:target} hold and that $A(\x) = k(\x,\x)$ for all $\x \in \mathcal{X}$. As $J \to \infty$, the empirical measure $\hat{q}^J_t(\x) \coloneqq \frac{1}{J} \sum_{j=1}^J \delta(\x - \x_t^{j})$ of particles evolving according to \eqref{eq:skmd_cts} converges weakly to $q_t$, the solution to
\begin{equation}
    \label{eq:mfl}
    \frac{\partial}{\partial t} q_t(\x) = \nabla \cdot \Bigg( q_t(\x) \int_{\mathcal{X}} k(\x,\x') \big( \nabla_{\x'} q_t(\x') + q_t(\x') \beta \, \nabla_{\x'} V_\theta(\x') \big) \upd \x' \Bigg) \ .
\end{equation} 
\end{proposition}

The stationary condition $\partial_t q_t = 0$ is satisfied by
$q_t = \pi_\theta \propto \exp(-\beta V_\theta)$, since
    \begin{align}
    \beta\nabla V_\theta(\x')\pi_\theta(\x') + \nabla\pi_\theta(\x') 
    & = \pi_\theta(\x')\left[\beta\nabla V_\theta(\x') + \nabla\log\pi_\theta(\x')\right] \notag \\
    & = \pi_\theta(\x')\left[\beta\nabla V_\theta(\x') - \beta\nabla V_\theta(\x')\right] = 0. \notag
    \end{align}

Therefore, the mean field limit of \eqref{eq:skmd_cts} has a stationary solution at $\pi_\theta$. \eqref{eq:mfl} does not depend on the scaling $\eta$, meaning that the diffusion coefficient can be of arbitrary magnitude as long as it vanishes as $J \to \infty$.

Note that \Cref{prop:mfl} is a finite-horizon consistency result. It shows that for $t \geq 0$, the empirical measure $\hat{q}^J_t$ converges weakly to $q_t$ in the large-particle limit $J \to \infty$. While $\pi_\theta$ is a stationary solution of \eqref{eq:mfl}, \Cref{prop:mfl} does not by itself establish convergence in the long-time limit. To prove such convergence results, we defer to existing results established by way of the mean-field limit for SVGD, which we briefly review as follows. Qualitative convergence, i.e., $q_t$ weakly converges to $\pi_\theta$ as $t \to \infty$, was established in \cite[Thm. 2.8]{Lu2019} under sufficient conditions on the kernel and the target measure. For quantitative rates of convergence, \cite{Duncan2023} shows that since the KL functional is not geodesically convex in the Stein geometry induced by the kernel, in general one cannot establish functional inequalities which would prove global exponential convergence to the target measure. They do, however, prove local exponential convergence \cite[Lemma 28]{Duncan2023} under the assumptions that one starts close to the target measure and has local strong convexity at equilibrium.

\subsection{Proof of \Cref{prop:mfl}}
\label{app:proof}

\begin{proof}
    The proof closely follows that of Proposition 2 in \cite{Duncan2023}, with a modification to the diffusion coefficient of the SDE. In the following, let
    \[
    b(\x,q) \coloneqq \int_{\mathcal{X}} \Big[ -k(\x, \x') \beta  \nabla_{\x'} V(\x') + \nabla_{\x'}k(\x,\x') \Big] \upd q(\x')
    \]
    and
    \[
    \sigma(\x) \coloneqq \sqrt{\frac{2\eta}{J}}, \quad \Sigma(\x) \coloneqq \sigma(\x)\sigma(\x)^\top = \frac{2 \eta}{J} \ .
    \]
    We use the duality principle with test functions to prove convergence of $\hat{q}^J_t$ to $q_t$, where it suffices to show that for a smooth test function with compact support $\phi \in C^\infty_c([0,\infty) \times \mathcal{X})$, $\int_{\mathcal{X}} \phi(t,\x) \upd \hat{q}^J_t(\x) \to \int_{\mathcal{X}} \phi(t,\x) \upd q_t(\x)$ as $J \to \infty$. 

    Let $\bar{\x}_t = \{\x^j_t\}_{j=1}^J$ denote the set of all particles whose empirical distribution correspon\upd s to $\hat{q}^J_t$ and $\Phi(t, \bar{\x}_t) \coloneqq \frac{1}{J} \sum_{j=1}^J \phi(t, \x_t^j) = \la \phi(t, \cdot), \hat{q}^J_t \ra$ denote the aggregate test function, which averages the test function evaluated at all the particles. Here, the brackets denote the duality pairing between test functions and measures. By It\^o's formula, we have that 
    \begin{align}
    \label{eq:itoformula}
    \upd \Phi(t, \bar{\x}_t) &= \frac{1}{J}\sum_{j=1}^J \left[\partial_t\phi(t, \x^j_t) 
    + \nabla \phi(t, \x^j_t) \cdot b(\x^j_t, \hat{q}^J_t)\right] \upd t \notag \\
    & \quad + \frac{\eta}{J}\operatorname{Tr}\left(\operatorname{Hess}\Phi(t,\bar{\x}_t)\right)\upd t + \upd \mathcal{M}_t, 
\end{align}
where $\upd \mathcal{M}_t \coloneqq \frac{\sqrt{2\eta}}{J \sqrt{J}} \sum_{j=1}^J \nabla \phi(\x^j_t) \upd W^j_t$ is a local martingale. The Hessian $\operatorname{Hess}\Phi \in \mathbb{R}^{Jn\times Jn}$ is block diagonal with blocks
    \[
    [\operatorname{Hess}\Phi(\bar{\x})]_{ij} = \begin{cases}
        \dfrac{1}{J}\operatorname{Hess} \phi(\x^i) & i = j, \\[6pt]
        0 & i \neq j.
    \end{cases}
    \]
Therefore, the It\^{o} correction evaluates to
    \[
    \frac{\eta}{J}\operatorname{Tr}\left(\operatorname{Hess}\Phi(t,\bar{\x}_t)\right)
    = \frac{\eta}{J^2}\sum_{j=1}^J \Delta\phi(\x^j_t) = \frac{\eta}{J} \int_{\mathcal{X}}\Delta\phi(\x)\, \upd \hat{q}^J_t(\x).
    \]
Integrating \eqref{eq:itoformula} from $0$ to $t$, we obtain
    \begin{align}
        \langle\phi(t,\cdot),\hat{q}^J_t\rangle - \langle\phi(0,\cdot),\hat{q}^J_0\rangle
        &= \int_0^t\langle\partial_s\phi(s,\cdot),\hat{q}^J_s\rangle\,\upd s \notag \\
        &\quad + \int_0^t\langle\nabla\phi(s,\cdot)\cdot b(\cdot,\hat{q}^J_s),
        \hat{q}^J_s\rangle\,\upd s \notag \\
        &\quad + \frac{\eta}{J}\int_0^t \la
        \Delta\phi(\cdot),\hat{q}^J_s \ra \upd s + \mathcal{M}_t.
    \end{align}
The quadratic variation of the local martingale is
    \begin{align}
        [\mathcal{M}_\cdot, \mathcal{M}_\cdot]_t & = \frac{2\eta}{J^3} \sum_{i,j=1}^J
        \int_0^t \nabla \phi(\x^i_s) \cdot \nabla \phi(\x^j_s) \, \upd s \notag \\ 
        & = \frac{2\eta}{J}\int_0^t\int_{\mathcal{X}} \nabla\phi(\x) \cdot \nabla\phi(\x')\,d\hat{q}^J_s(\x) d\hat{q}^J_s(\x') \,\upd s, \notag
    \end{align}
which is $O(J^{-1})$. Therefore, $\mathcal{M}_t \to 0$ in probability as $J \to \infty$.

Assume that the family $\{\hat{q}^J_\cdot : J \in \mathbb{N}\}$ has a limit point $q_\cdot \in \mathcal{P}(C[0,T])$. Formally, as $J \to \infty$, the It\^{o} correction and quadratic variation of the local martingale both vanish,
yielding
\begin{equation}
    \langle\phi(t,\cdot),q_t\rangle - \langle\phi(0,\cdot),q_0\rangle
    = \int_0^t\langle\partial_s\phi(s, \cdot),q_s\rangle\,\upd s
    + \int_0^t\langle\nabla\phi(s, \cdot)\cdot b(\cdot,q_s),q_s\rangle\,\upd s ,
\end{equation}
which is the weak formulation of the nonlinear transport equation
\begin{equation}
    \partial_tq_t(\x) = -\nabla\cdot(b(\cdot,q_t) \, q_t).
\end{equation}
Since $\nabla_{\x'} k(\x,\x') = -\nabla_\x k(\x,\x')$ by symmetry of the kernel and $\nabla q_t = -\beta\nabla V_\theta \cdot q_t
+ \nabla q_t$ from integration by parts, we recover \eqref{eq:mfl} after substituting the definition of $b$.
\end{proof}

\section{Algorithm Details and Variants}
\label{app:algorithm}

\subsection{SKMD for Boltzmann sampling}
\label{app:samp}

SKMD can be used as a general purpose algorithm for sampling the invariant distribution of an energy-based system. \Cref{alg:samp} corresponds to an asynchronous implementation of stochastic SVGD, where each particle is evolved one at a time for a fixed number of steps $\ell$, such that the interacting particle dynamics may be approximated by a single chain. As shown in \Cref{app:analysis}, the distribution of the dynamics defined by a potential $V$ in the continuous time limit ($\epsilon \to 0$), synchronous limit ($\ell \to 0$), and mean-field limit ($J \to \infty$) approaches the invariant distribution $\pi \propto \exp(-V)$, assuming $\beta=1$ without loss of generality. 

\begin{algorithm}
\caption{SKMD for Boltzmann sampling}
\label{alg:samp}
\begin{algorithmic}
    \Require{Initial ensemble $\bar{X}_0 = \{\x^{1}_0, \ldots, \x^{J}_0\}$, trajectory $X = \emptyset$, potential $V$, kernel $k$, stopping $\ell$} 
    
    \State Initialize ensemble index $s \gets 0$, active particle index $i \gets 1$
    \For{$t=0,1,2,\ldots$}
        \State Simulate one step of \eqref{eq:skmd}: $\x^i_{t+1} \gets \text{SKMD}(\x^i_t, \bar{X}^{-i}_s, V, k)$

        \If{$t+1=s+\ell$}
            \State Update ensemble: $\bar{X}_{s+\ell} \gets \bar{X}^{-i}_s \cup \{\x^{i}_{s+\ell}\}$ 
            \State Append to trajectory: $X \gets X \cup \{\x^i_s,...,\x^i_{s+\ell} \}$
            \State Advance ensemble index: $s \gets s + \ell$
            \State Switch active particle: $i \gets (i \bmod J) + 1$
        \EndIf
    \EndFor
            
    \State \Return Trajectory $X$ with approximate distribution $\pi \propto \exp(- V)$
\end{algorithmic}
\end{algorithm}

\subsection{SKMD for active learning with offline data acquisition}
\label{app:skmd_al}

SKMD may be paired with subset selection methods for offline data acquisition. In this setting, SKMD is used purely as a sampling algorithm for generating a pool of candidate configurations. We then rely on any subset selection method, such as the ones reviewed in \Cref{sec:intro}, to draw sets of informative data to label with reference calculations and add to the training set. This approach is summarized in \Cref{alg:al_offline}. 

\begin{algorithm}
\caption{SKMD for active learning with offline data acquisition}
\label{alg:al_offline}
\begin{algorithmic}
    \Require{Initial ensemble $\bar{X}_0 = \{\x^{1}_0, \ldots, \x^{J}_0\}$, training set $\mathcal{D}_0$, model $V_{\theta_0}$, kernel $k$, stopping $\ell$} 
    
    \For{training iteration $\tau = 0,1,\ldots$}
        \State Generate SKMD trajectory $X$ from \Cref{alg:samp} using $\bar{X}_0$, $V_{\theta_0}$, $k$, $\ell$
        \State Select training data $\mathcal{D}_{\tau+1} \subset X$ using any subset selection method
        \State Train new model $V_{\theta_{\tau+1}}$ on the augmented training set $\bigcup_{k=0}^{\tau+1} \mathcal{D}_{k}$
    \EndFor
    \State \Return Final model $V_{\theta_\text{final}}$ and training set $\mathcal{D}_\text{final}$
\end{algorithmic}
\end{algorithm}

A common heuristic with SVGD is to utilize a kernel with a variable kernel bandwidth, such as one which is proportional to the median distance between particles, as a strategy to accelerate mixing in high-dimensional state space. When implementing SKMD with a variable kernel bandwidth, offline data acquisition tends to perform better than online data acquisition. The acquisition criterion in \eqref{eq:stopping} remains relatively constant for the median bandwidth heuristic, since the kernel bandwidth adapts such that the kernel terms in the SKMD biasing force are nonzero and the particles remain interactive at further distances. For this reason, in \Cref{sec:mace}, we utilize the median bandwidth heuristic and farthest-point sampling to select training configurations for active learning of the MACE interatomic potential.

\section{Additional Background and Related Work}
\label{app:review}

\subsection{Enhanced sampling methods}
\label{app:enhancedsampling}

Biasing force methods are a class of enhanced samplers which modify the force field of the dynamics by the addition of a biasing force; see \cite{Henin2022} for a review. In these methods, the force field in overdamped or underdamped Langevin dynamics is replaced by
\begin{equation}
    \label{eq:biasforce}
    \tilde{F}_{\theta,t}(\x_t) = -\nabla_\x V_\theta(\x_t) + F_t^{\text{bias}}(\x_t) 
\end{equation}
%
assuming $\beta=1$ without loss of generality. In the standard formulation of \textbf{adaptive biasing forces (ABF)} \cite{Darve2001}, the objective is to reduce energy barriers characterized in a reduced coordinate space of collective variables (CVs) $\xi(\x) \in \mathbb{R}^r$, where $r \ll 3N$. At a given point $\x \in \mathbb{R}^{3N}$, ABF adaptively learns and cancels gradient of the free energy in CV space by setting the biasing force to be
\begin{equation}
    \label{eq:abf}
    F^\text{ABF}_{t}(\x_t) = \nabla_\x \xi(\x)^\top \nabla_{\xi} \mathcal{F}(\xi(\x)),
\end{equation}
where $\nabla_\x \xi$ is the Jacobian of the CV map. The resulting dynamics correspond to uniform sampling of the CV space and enhanced exploration of state space regions separated by high-energy barriers. The performance of ABF depends heavily on the identification of a reliable CV map; otherwise, the method is not guaranteed to improve sampling. Various other methods in the same vein as adaptive biasing force methods have been proposed in the statistics literature which preserve the target invariant measure while altering transient behavior to accelerate convergence; see \cite{Girolami2011,Duncan2017} for discussion on reversible and irreversible perturbations for MCMC. 

Adaptive biasing potential methods are a subclass of adaptive biasing force methods which are obtained when the biasing force is conservative, meaning that it can be written as the gradient of a potential, $F_t^{\text{bias}} = -\nabla_\x V_t^{\text{bias}}$. In this case, the biased dynamics are described by a modified potential energy,
\begin{equation}
    \label{eq:biasenergy}
    \tilde{V}_{\theta,t}(\x_t) = V_\theta(\x_t) + V_t^{\text{bias}}(\x_t) \ .
\end{equation}
The form of the biasing potential $V_t^{\text{bias}}: \mathbb{R}^{3N} \to \mathbb{R}$ depends on the sampling objective. In \textbf{metadynamics} \cite{Barducci2008, Valsson2016}, sampling is promoted along selected directions of the state space defined by CVs in order to characterize the free energy landscape or rare-event kinetics. The biasing potential in metadynamics is a sum of repulsive Gaussian kernels which are deposited at times $\mathcal{S}_\tau = (s_1,...,s_\tau)$ along the simulation path in CV space,

\[
V_t^{\text{meta}}(\x_t)=\sum_{s \in \mathcal{S}_\tau} a_s \exp \Big( - \frac{||\xi(\x_s) - \xi(\x_t)||^2}{2b_s^2} \Big) \ ,
\]
where $a_s, b_s > 0$ are the magnitude and length scale parameters of the Gaussian kernel at time $s$. Metadynamics acts to ``fill'' visited regions with repulsive kernels until the modified free energy landscape is leveled and sampling becomes approximately uniform in the CV space. As with ABF, metadynamics relies on the availability of an informative collective variable. 

In \textbf{uncertainty-driven dynamics (UDD)} \cite{Kulichenko2023}, a biasing potential is introduced to enhance the sampling of uncertain regions of the free energy landscape for the purpose of active learning. Given a committee of model potentials $(\hat{V}^{(m)}_t)_{m \in [1,M]}$ and their mean prediction $\bar{V}_t = \frac{1}{M} \sum_{m=1}^{M} \hat{V}^{(m)}_t$,  a biasing potential is constructed as
\[
V_t^{\text{UDD}}(\x_t) = a \Big[ \exp \Big( -\frac{ \sigma^2_V(\x_t)}{NMb^2} \Big) - 1 \Big] \ ,
\]
where $a,b > 0$ are scale parameters and $\sigma_V^2$ is the empirical variance of energy predictions by models in the committee,
\begin{equation}
\label{eq:empvariance}
\sigma_V^2(\x_t) = \frac{1}{M} \sum_{m=1}^{M} (\hat{V}^{(m)}_t(\x_t) - \bar{V}_t(\x_t))^2 \ .
\end{equation}
The model committee in \cite{Kulichenko2023} is constructed via $k$-fold cross validation, with each model trained on a distinct partition of the dataset. 

In most cases, the biasing potential alters the invariant measure of the dynamics. Assuming the biasing potential remains time-invariant or approaches a stationary limit, $\lim_{t \to \infty} V^\text{bias}_t = V^\text{bias}$, the state dynamics converge to a biased Boltzmann distribution, 

\begin{equation}
\tilde{\pi}_\theta(\x) = \frac{1}{\hat{Z}} \exp (-\tilde{V}_{\theta}(\x)) = \frac{1}{\hat{Z}} \exp \big(-V_\theta(\x) - V^\text{bias}(\x)) \ .
\end{equation}
Samples from $\tilde{\pi}_\theta$ may still be used to calculate summary statistics of the system under the original Boltzmann distribution $\pi_\theta$ through an appropriate importance reweighting scheme. For some measurable function of positions $f(\x)$, its expectation with respect to $\pi_\theta$ can be computed as

\begin{equation}
    \mathbb{E}_{\x \sim \pi_\theta}[f(\x)] = \mathbb{E}_{\x \sim \tilde{\pi}_\theta} \big[ f(\x) w(\x) \big] \ ,
\end{equation}
where $w(\x)$ is the self-normalized importance sampling reweighting function given by 

\begin{equation}
    w(\x) = \frac{\omega(\x)}{\int \omega(\x') \tilde{\pi}_\theta(\x') \upd \x'}, \quad \omega(\x) = \frac{ \exp(- V_\theta )}{\exp(-V_\theta - V^\text{bias})} = \exp(V^\text{bias}) \ .
\end{equation}
Note that $\omega$ is the ratio of the unnormalized forms of $\pi_\theta$ and $\tilde{\pi}_\theta$. Since the reweighting function depends exponentially on the biasing potential, even moderate distortions of the biasing potential can induce highly variable weights, as well as high variance or numerical instability of the importance sampling estimator. 

\subsection{Variational inference and gradient flows}
\label{app:varinf}

In various sampling and generative modeling problems, the goal is to approximate a target density $\pi$ on $\mathcal{X}$ with a set of samples. Variational inference is a class of techniques which find an optimal density $q^*$, typically from a parametric family $\mathcal{Q}$, which minimizes the Kullback-Leibler (KL) divergence from the target density, 
    \begin{equation}
    \label{eq:varinf}
    q^* \in \underset{q \in \mathcal{Q}}{\text{argmin}} \ \mathbb{D}_{\textup{KL}}(q || \pi) \ .
    \end{equation}
Recently, a class of particle-based variational inference techniques have emerged which approximate a target density with the empirical distribution of a dynamic set of particles. The core principle behind these methods is to view variational inference as a gradient flow of KL divergence on the space of probability distributions. The evolution of a distribution $q$ over time is given by the continuity equation,
\begin{equation}
    \label{eq:ctyeqn}
    \frac{\partial q_t}{\partial t} = -\nabla \cdot (q_t \phi_t) \ ,
\end{equation}
where $\phi_t$ is a velocity field moving probability mass. One can discretize the flow in \eqref{eq:ctyeqn} with particles whose dynamics are governed by
\begin{equation}
\frac{\upd \x_t}{\upd t} = \phi_t (\x_t), \quad \x_0 \sim q_0,
\end{equation}
such that $q_t = \text{Law}(\x_t)$. The velocity field $\phi_t$ is not uniquely determined a priori, but depends on the choice of geometric structure on the space of probability measures \cite{Ambrosio2008}. For instance, the law of overdamped Langevin dynamics evolves according to the Fokker-Planck equation, which can be written in the form of a continuity equation and interpreted as the Wasserstein gradient flow of KL divergence \cite{Liu2017}.

\textbf{Stein variational gradient descent (SVGD)}, introduced in \cite{Liu2016}, is a particle-based variational inference algorithm which corresponds to a gradient flow restricted to velocity fields that lie in a reproducing kernel Hilbert space (RKHS). This constraint yields a closed-form expression for the steepest descent direction of KL divergence within the RKHS and a tractable particle-based algorithm. 

In SVGD, one solves \eqref{eq:varinf} by constructing a sequence of empirical measures $(\hat{q}_0, \hat{q}_1, ..., \hat{q}_t)$ which converges weakly on the target distribution $\pi$. Beginning with the empirical measure $\hat{q}_0(\x) = \frac{1}{J} \sum_{j=1}^J \delta(\x - \x^{j})$ of an initial set of particles $\bar{X}_0 = \{\x_0^{j} \}_{j=1}^J$, we take the empirical measure at iteration $t$ to be $\hat{q}_t = T_{t \sharp} \hat{q}_0$, or the pushforward of $\hat{q}_0$ under some transport map $T_t$. Equation \eqref{eq:varinf} then becomes the minimization of discrepancy between $T_{t \sharp} q_0$ and the target distribution, 
    \begin{equation}
    \label{eq:svgd_transform}
    T_t^* \in \underset{T_t}{\text{argmin}} \ \mathbb{D}_{\textup{KL}}(T_{t \sharp} q_0 || \pi) \ ,
    \end{equation}
where $q_0$ is the mean-field measure corresponding to $\hat{q}_0$. The transport map at iteration $t$ is composed of intermediate maps $T_t = \bar{T}_t \circ \bar{T}_{t-1} \circ ... \circ \bar{T}_1$ which take the form $\bar{T}_t(\x) = \x + \epsilon \phi_t(\x)$. The transformation consists of perturbing the positions of particles at time $t$ by a step size $\epsilon$ in the direction $\phi_t$, which belongs to a class of bounded functions $\mathcal{F} \coloneqq \{ f : ||f||_\mathcal{F} \leq 1 \}$. We can then reframe the optimization problem to be one of choosing $\phi_t$ to be the direction of steepest gradient descent, maximally decreasing KL divergence with respect to the target distribution, i.e.
    \begin{equation}
    \label{eq:svgd_perturb}
    \phi_t^* \in \underset{\phi_t \in \mathcal{F}}{\text{argmax}} \ \big[ -\nabla_\epsilon \mathbb{D}_{\textup{KL}}(T_{t\sharp} q_0 || \pi) |_{\epsilon=0} \big] \ . 
    \end{equation}
It is shown in \cite{Liu2016, Liu2017} that a practical approach to solving the above optimization problem is to project $\phi_t$ onto a reproducing kernel Hilbert space (RKHS). We choose $\mathcal{F}$ to be the RKHS $\mathcal{H}^n$ for $\mathcal{X} \subseteq \mathbb{R}^n$, defined by a symmetric positive semi-definite kernel $k: \mathcal{X} \times \mathcal{X} \to \mathbb{R}$. Then we have a closed-form expression for the gradient of KL divergence, given by
    \begin{subequations}
        \label{eq:stein}
    	\begin{gather}
    	\nabla_\epsilon \mathbb{D}_{\textup{KL}}( q_t || \pi) |_{\epsilon=0} = - \langle \phi_t, \psi_{t} \rangle_{\mathcal{H}^n} \\
        \label{eq:psi}
    	\psi_{t}(\cdot) = \mathbb{E}_{\x \sim q_t} [k(\x, \cdot) \nabla_{\x} \log \pi(\x) + \nabla_{\x} k(\x, \cdot)] \ .
    	\end{gather}
    \end{subequations}
It follows that the solution to \eqref{eq:svgd_perturb}, giving the steepest descent direction at  step $t$, is $\phi_t^* \propto \psi_{t}$.

The flow of empirical measures corresponds to the state evolution of the interacting particle set, represented by the ordinary differential equation
    \begin{subequations}
        \label{eq:steinode2}
        \begin{gather}
	\upd \x^{i}_t = \phi^*_t(\x^{i}) \upd t, \\
    \label{eq:v_svgd}
    \phi^*_t(\cdot) = \frac{1}{J} \sum_{j=1}^J \Big[ k(\x^{j}_t, \cdot) \nabla_{\x^{j}_t} \log \pi(\x^{j}_t) + \nabla_{\x^{j}_t} k(\x^{j}_t, \cdot) \Big],
        \end{gather}
    \end{subequations}
for all $\x^{i}_t,\x^{j}_t \in \bar{X}_t$. As $J \to \infty$, the mean field limit of \eqref{eq:steinode2} is the following, as shown in \cite{Liu2017,Duncan2023}:
\begin{equation}
    \label{eq:meanfieldlimit}
    \frac{\partial}{\partial t} q_t(\x) = \nabla \cdot \Bigg( q_t(\x) \int_{\mathcal{X}} k(\x,\x') \big( \nabla_{\x'} q_t(\x') - q_t(\x') \nabla_{\x'} \log \pi(\x') \big) \upd \x' \Bigg) \ .
\end{equation}
 Under appropriate assumptions on the kernel, target measure, and initial condition, the solution to \eqref{eq:meanfieldlimit} exists and is unique \cite[Thm. 3.2]{Lu2019}.
 
 It is straightforward to verify that $\pi$ is a stationary point of \eqref{eq:meanfieldlimit}, where $\frac{\partial q_t}{\partial t} = 0$ when $q_t = \pi$. A proof of qualitative convergence was provided in \cite[Thm. 2.8]{Lu2019}, which showed the conditions under which $q_t$ converges weakly to $\pi$ as $t \to \infty$. To prove convergence with quantitative rates, the authors of \cite{Duncan2023} first recognize that that SVGD defines a gradient flow of KL divergence not in the Wasserstein geometry, but in the Stein geometry induced by the kernel. In this geometry, the KL functional is not geodesically convex, meaning that in general one cannot establish functional inequalities that would prove global exponential convergence to the target measure. However, one can prove exponential convergence provided that one has local strong convexity and starts close to the target measure at equilibrium \cite[Lemma 28]{Duncan2023}.

\subsection{Kernel herding}
\label{app:herding}

Kernel herding \cite{Chen2010} is an approach to constructing a discrete approximation of a probability distribution using a finite set of points, referred to as support points \cite{Mak2017,Mak2018} or coresets \cite{Dwivedi2021,DomingoEnrich2023,Moser2025} in the literature. For a kernel $k(\cdot,\cdot)$ which defines a reproducing kernel Hilbert space $\mathcal{H}$, the kernel mean embedding of a probability distribution $\pi$ in $\mathcal{H}$ is defined as
\[
\mu_\pi \coloneqq \mathbb{E}_{\x \sim \pi}[k(\x, \cdot)] \ .
\]
The goal is then to approximate $\mu_\pi$ with an empirical embedding 
\[
\hat{\mu}_N = \frac{1}{N} \sum_{i=1}^N k(\x_i, \cdot) \ ,
\]
defined by $N$ points which minimize the maximum mean discrepancy $\text{MMD}^2(\hat{q}_N, \pi) \coloneqq || \hat{\mu}_N - \mu_\pi ||^2_\mathcal{H}$, where $\hat{q}_N = \tfrac{1}{N} \sum_{j=1}^N \delta(\x - \x^j)$ is the empirical distribution of the point set. Since global minimization of the objective over point sets is generally intractable, kernel herding finds a greedy solution to the minimization problem: for $n=1,...,N$, given a set of previously chosen points $\{ \x_i \}_{i=1}^n$, kernel herding selects the next point as the one which decreases $\text{MMD}^2(\hat{q}_n, \pi)$ according to
\begin{equation}
    \label{eq:kernelherd}
    \x_{n+1} \in  \underset{x \in \mathcal{X}}{\text{argmin}} \Bigg[ \frac{1}{n} \sum_{i=1}^n k(\x_i, \x) - \mathbb{E}_{\x' \sim \pi}[k(\x', \x)] \Bigg] \ .
\end{equation}

\textbf{Stein herding} \cite{Chen2018} is an analogous approach to kernel herding which targets the minimization of kernel Stein discrepancy. Let $\mathcal{S}_\pi$ be the Stein operator where
\[
\mathcal{S}_\pi \otimes k(\x, \cdot) \coloneqq \nabla_\x \log \pi(\x)k(\x, \cdot) + \nabla_\x k(\x, \cdot) \ .
\]
Kernel Stein discrepancy is defined as $\textup{KSD}^2(\hat{q}_n, \pi) \coloneqq \mathbb{E}_{\x,\x' \sim \hat{q}_n}[\kappa_\pi(\x,\x')]$, where $\kappa_\pi: \mathcal{X} \times \mathcal{X} \to \mathbb{R}$ is the Stein kernel
\[
    \begin{aligned}
    \kappa_{\pi}(\x,\x') & \coloneqq \mathcal{S}^{\x}_{\pi} \mathcal{S}^{\x'}_{\pi} \otimes k(\x, \x') \\
    & = \nabla_\x \cdot \nabla_{\x'} k(\x, \x') \\
    & \quad + \nabla_\x k(\x, \x') \cdot \nabla_{\x'} \log \pi(\x') \\
    & \quad + \nabla_{\x'} k(\x, \x') \cdot \nabla_{\x} \log \pi(\x) \\
    & \quad + k(\x, \x') \nabla_{\x} \log \pi(\x) \cdot \nabla_{\x'} \log \pi(\x') \ .
    \end{aligned}
\]

In Stein herding, the next point which minimizes $\textup{KSD}^2(\hat{q}_n, \pi)$ is given by
\begin{equation}
    \label{eq:steinherd}
    \x_{n+1} \in  \underset{x \in \mathcal{X}}{\text{argmin}} \Big[ \frac{1}{n} \sum_{i=1}^n \kappa_{\pi}(\x_i, \x) \Big] \ ,
\end{equation}
which is identical to the expression in \eqref{eq:steinpts}. As noted in \cite{Chen2018}, each $n$th iteration of Stein herding or kernel herding requires the solution to a global optimization problem over $\mathcal{X}$, which the authors perform using numerical optimization methods such as the Nelder-Mead algorithm and brute search. 

\begin{remark}
    \label{rmk:acquisition}
    The acquisition function in \eqref{eq:stopping} can be compared to the greedy update equation \eqref{eq:steinherd} for Stein points as follows. The SKMD biasing force \eqref{eq:skmd_force} is derived from the velocity field of SVGD \eqref{eq:steinexpec}, which corresponds to 
    \[
    \phi^*_{\hat{q}_n,\pi}(\cdot) = \mathbb{E}_{\x \sim \hat{q}_n}[\mathcal{S}_{\pi} \otimes k(\x, \cdot)] = \mathbb{E}_{\x \sim \hat{q}_n}[-\beta \nabla_\x V (\x) k(\x, \cdot) + \nabla_\x k(\x, \cdot)] \
    \]
    for $\pi \propto \exp(-\beta V)$. Kernel Stein discrepancy can be expressed as as the RKHS norm of the velocity field,  
    \[
    \textup{KSD}^2(\hat{q}_n, \pi) = \mathbb{E}_{\x,\x' \sim \hat{q}_n}[\mathcal{S}^{\x}_{\pi} \mathcal{S}^{\x'}_{\pi} \otimes k(\x, \x')] = ||\phi^*_{\hat{q}_n,\pi}||^2_\mathcal{H} .
    \]
    The acquisition function $\alpha_s$ in \eqref{eq:stopping} corresponds to a pointwise Euclidean norm of the velocity field, i.e., $\alpha_s(\x) = ||\phi^*_{\hat{q}_n,\pi}(\x)||$. Therefore, the acquisition criterion does not relate directly to the minimization of KSD, as the objective differs in terms of the Hilbert space norm. In practice, the Euclidean norm of $\phi^*_{\hat{q}_n,\pi}$ is simpler to compute in an online fashion compared to the RKHS norm, as it does not require the calculation of additional gradients or second-order terms at each simulation step. For future work, we plan to expand upon our initial studies in evaluating the efficacy of kernel Stein discrepancy in defining an adaptive acquisition criterion.
\end{remark}

\section{Additional Discussion}
\label{app:discussion}

\subsection{Online data acquisition}
\label{app:adaptstop}

In the numerical studies of \Cref{sec:nn}, we compare the online data acquisition schemes of UDD and a-SKMD. UDD utilizes an acquisition function defined as the scaled standard deviation of committee predictions,
\begin{equation}
\label{eq:uncertainty}
\rho(\x) \coloneqq \sqrt{\tfrac{2}{M}} \sigma_V(\x) \ ,
\end{equation}
where $\sigma_V$ is defined in \eqref{eq:empvariance}. At the point where the uncertainty metric $\rho$ exceeds a threshold, i.e., $\rho(\x_t) > \zeta_1$, the point is collected as new training datum. 

\Cref{fig:adaptstop} compares the acquisition functions of a-SKMD and UDD after one iteration of active learning. The realization of Brownian motion for the path simulation is fixed between the two schemes shown. The SKMD acquisition function, defined as the norm of the SKMD biasing force, is informed by both the underlying potential and the kernel. Panel (c) of \Cref{fig:adaptstop} shows that the SKMD acquisition function is low near other particles in the ensemble which lie in energy basins, as a result of the attractive component of the biasing force, while the function is high near other particles located in high energy regions, as a result of the repulsive component of the biasing force. Therefore, the criterion serves to select points which are diverse with respect to the other particles but also representative of the underlying energy landscape. Moreover, the threshold of the SKMD criterion spans a large area of state space, such that with each active learning iteration, the criterion leads to a well-defined radial expansion in state space of selected points. Panel (d) shows that the UDD acquisition function increases with distance from the existing training data, with regions of high uncertainty concentrated in certain domains of the state space. The location of these regions does not correlate with the underlying energy landscape. Moreover, the UDD criterion does not have a means to enforce diversity among points selected at a given active learning iteration, which can lead to clustering in the high-uncertainty regions. As a result, the data selected by the SKMD criterion often exhibit greater spread in the state space compared to those by the UDD criterion, as shown in panels (a) and (b).

\begin{figure}[h!]
    \centering
    \includegraphics[width=0.95\linewidth]{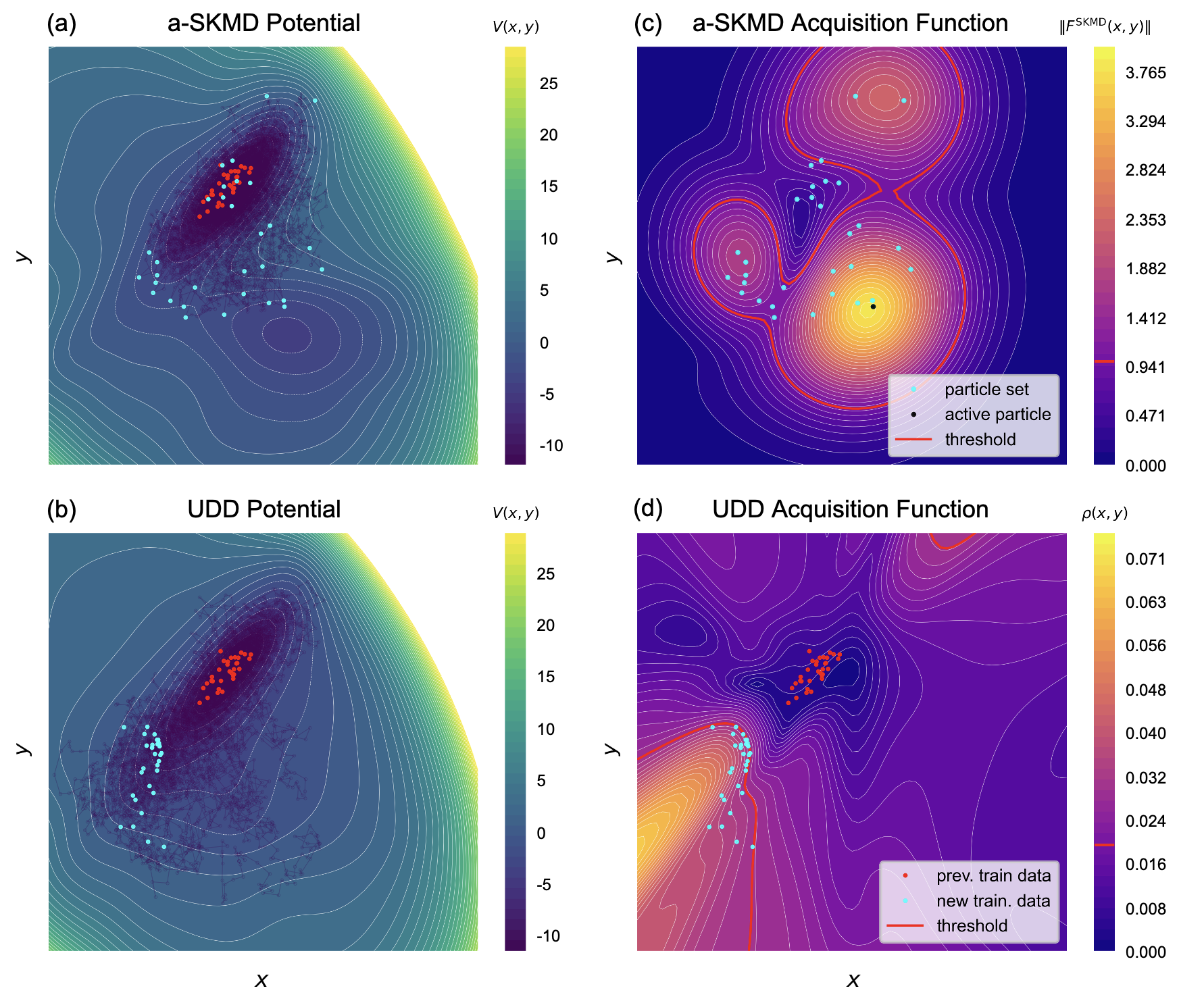}
    \caption{\small{The neural network potential (contours), previous training data (red), and selected data (cyan) corresponding to 1 iteration of active learning by a-SKMD (a) and UDD (b). The corresponding contours of the acquisition function, selected data (cyan), and threshold for the acquisition criterion (red line) for a-SKMD (c) and UDD (d).}}
    \label{fig:adaptstop}
\end{figure}

\section{Experimental Details}
\label{app:experiment}

\subsection{Neural network potential}
\label{app:exp_nn}

\textbf{System and model.} The neural network (NN) potential is a multi-layer perceptron defined by a single 20-dimensional hidden layer, tanh activation functions, and a confining quartic term to ensure that the system has a well-defined Boltzmann distribution. We use autodifferentiation to compute the negative gradient of the neural network potential. The initial training set consists of 32 points in one metastable basin.

\textbf{Active learning protocol.} For all active learning schemes in the comparison, we fix the temperature parameter to be $\beta = 1.0$ and perform numerical simulation with a time step $\epsilon=0.002$. In each trial, a maximum of 16 active learning iterations are performed and 32 new data points are queried in each active learning iteration.  

We implement UDD with a committee of 5 models, each trained on a $k$-fold subset of the data, using the uncertainty-based acquisition criterion \eqref{eq:uncertainty} specified in \cite{Kulichenko2023}. From a pilot simulation, we determined $\zeta_1=0.02$ to be an appropriate uncertainty threshold. In order to reduce redundancy in the collected points, we further impose that UDD performs a minimum of 50 simulation steps before the next datum is collected. 

a-SKMD utilizes the acquisition function defined by the norm of the SKMD biasing force in \eqref{eq:stopping}, requiring a minimum number of simulation steps $\ell_0=50$ before stopping. We determine an appropriate value for the threshold is $\zeta_0=1.0$ by tracking the acquisition function $\alpha_{s}$ over a pilot simulation.

In the na\"ive data acquisition scheme utilized in overdamped Langevin dynamics and SKMD without adaptive stopping, we collect data every 300 simulation steps and stop the simulation after 32 points have been collected. This corresponds to a roughly equal simulation budget of around $1.5 \times 10^5$ time steps/force calls in all of the evaluated methods. Note that a-SKMD and UDD will inherently have a variable number of simulation time steps between acquired points and training iterations as a result of the adaptive stopping rules. This is a key advantage of online data acquisition, that it automatically determines an appropriate query timing rather than requiring a user to determine it \textit{a priori}.

\textbf{SKMD parameters.} SKMD is implemented according to \Cref{alg:al_offline} using $\ell=100$. a-SKMD is implemented according to \Cref{alg:al_online}. Both SKMD and a-SKMD utilize $A=1$ and a Gaussian kernel defined by the Euclidean distance in state space with amplitude $a=2$ and fixed length scale $b=0.41$, which is the median distance between elements in the initial ensemble. For both schemes, the interacting set of $J=32$ points is initialized at the initial training set. In each active learning iteration, all particles are simulated one at a time and the data are selected to be the stopping points of each particle. 

\textbf{Training protocol.} In every AL iteration across the methods, the NN potential is trained using the weighted least squares loss in \eqref{eq:loss} with $\lambda_0=0.5$ and $\lambda_1=1$. The loss is optimized with the Adam algorithm over 1000 epochs using a batch size of 8. 

\textbf{Evaluation.} Model error is evaluated in terms of the root mean squared error (RMSE) in potential energy and forces on a test set of $10^3$ samples drawn approximately i.i.d. from the Boltzmann distribution $\pi$ of the M\"uller--Brown potential.

\subsection{Alanine dipeptide}
\label{app:exp_mace}

\textbf{System and model.} The oracle is MACE-OFF-23-small \cite{Kovacs2025}, which is fit to the SPICE dataset \cite{Eastman2023}. The surrogate is a MACE model with 32 invariant channels, initially fit to a random 0.1\% of the training set of the oracle.

\textbf{Active learning protocol.} All samplers use $J=30$ particles, initialized from a burn-in simulation started at Structure~1, and an equal budget of 300 outer-loop steps per active learning iteration. In each outer-loop step, one particle is advanced by 500 MD steps, and in subsequent steps the particles are cycled in round-robin order. For the unbiased and offline SKMD runs, the 300 trajectory endpoints form the candidate set, from which 100 configurations are selected by FPS in descriptor space, taking all previously selected samples into account. a-SKMD instead acquires configurations online with the adaptive stopping criterion~(8), after a short exploration phase that sets $b$ and calibrates $\zeta_0$. uHAL acquires configurations online at peaks of the HAL selection score~\cite{vanderOord2022}. Both online methods stop once 100 configurations have been acquired or the budget is exhausted. The acquired configurations are labelled by the oracle and added to the fine-tuning dataset.

\textbf{Integrators.} The unbiased underdamped simulations use a timestep of 1~fs and Langevin thermostats with a damping parameter of 0.1~ps. SKMD, a-SKMD and the overdamped 50~K baseline use overdamped (Brownian) dynamics at 50~K, with a timestep of 0.1~fs and an isotropic translational viscous damping coefficient of 75.0~$\mathrm{g\,mol^{-1}\,ps^{-1}}$ applied to each atom. uHAL uses underdamped Langevin dynamics at 700~K with a timestep of 1~fs and a friction of $0.01~\mathrm{fs^{-1}}$, which is equivalent to the 0.1~ps damping of the unbiased runs.

\textbf{Initialization.} Structure~1 (Figure~\ref{fig:mace_fig_1}b) is relaxed with the surrogate potential and then equilibrated with a 4000-step burn-in simulation. The initial positions of the $J=30$ particles are taken from a further 6000 simulation steps, recording a snapshot every 200 steps.

\textbf{Training protocol.} The same training protocol was used for every active learning iteration and for all sampling methods. We fine-tune with Low-Rank Adaptation (LoRA)~\cite{hu2022lora}, in which the pretrained weights are adapted to new data through low-rank updates. At each iteration, the model was fine-tuned on the accumulated training set for 100 epochs with LoRA rank 4 and scaling factor $\alpha=1.0$. Training used a batch size of 8, 10\% of the data held out for validation, AMSGrad, an exponential moving average of the weights (decay 0.999) and double precision. The isolated atom energies were consistent between the pretrained model and the oracle-labelled dataset, so multi-head fine-tuning was disabled. Training used a two-stage schedule. During the first 50 epochs, the force and energy loss weights were 1000 and 40, with a learning rate of 0.01. From epoch 50 onward, they were 10 and 1000, with a learning rate of 0.001.

\textbf{SKMD parameters.} The SKMD implementation for the MACE examples evaluates the biasing force $\beta = 1$~eV$^{-1}$, as in the dimensionless M\"uller--Brown setting, rather than $\beta = 1/(k_B T)$. This is a constant-$\gamma$ instance of the driving-force weighting used in annealed SVGD~\cite{dangelo_annealed_2021}, favoring the repulsive, exploratory term. The kernel in SKMD uses an amplitude $a=3.0$. 

\textbf{SKMD outer-loop schedule.} The 300 outer iterations are split into an exploration phase (iterations 1--250) and a fixed-bandwidth phase (iterations 251--300). During exploration, the kernel amplitude is held fixed at $a=3.0$, and the bandwidth $b$ is set adaptively at each iteration to the median pairwise Euclidean distance between the descriptor vectors of all particles. At iteration 251, $b$ is frozen to the mean of its values over the preceding 20 iterations, and $a$ and $b$ are held constant thereafter.

\textbf{a-SKMD acquisition schedule.} A suitable kernel bandwidth $b$ and acquisition threshold $\zeta_0$ are difficult to know \textit{a priori} for a molecular system. a-SKMD therefore begins with an exploration phase of 50 outer-loop steps, in which $b$ is set adaptively to the median pairwise distance between particle descriptors and no configurations are acquired. Over the final 20 exploration steps, we record $b$ and the norm of the SKMD biasing force $\|F^{\mathrm{SKMD}}_{\theta,s}\|$ averaged over each 500-step segment. At the end of exploration, $b$ is fixed to the mean of the recorded bandwidths, and $\zeta_0$ is set to the 10th percentile of the recorded force norms. In every subsequent outer-loop step, each particle is simulated for at most 500 MD steps and stops as soon as the acquisition criterion~(8) is met, i.e.\ $\|F^{\mathrm{SKMD}}_{\theta,s}\| < \zeta_0$, subject to a minimum of $\ell_0 = 250$ steps to limit correlation between successive samples. The configuration at the stopping time is acquired for labelling. If the criterion is not met within 500 steps, the particle is advanced without acquiring a configuration. Both $a$ and $b$ are held fixed during acquisition. Each active learning iteration terminates once 100 configurations have been acquired, matching the number selected by the offline schemes, or once the budget of 300 outer-loop steps is exhausted. In practice, a-SKMD used on average 225 of the 300 outer-loop steps (range 166--300); in 5 of the 50 iterations the budget was exhausted with 46--97 configurations acquired.

\textbf{uHAL acquisition.} uHAL runs use ASE-uHAL~\cite{rocke2025aseuhal}. The model is a 10-member linear committee on the MACE descriptors, with prior weight 0.5 and energy and force weights 1 and 10. The committee is initialised from this prior at the start of each active learning iteration. The bias strength is set adaptively relative to the mean-model forces ($\tau_{\mathrm{rel}}=0.03$). Particles are advanced in the same round-robin schedule as the other samplers, and the HAL selection score~\cite{vanderOord2022} is evaluated every 2 MD steps. A configuration is acquired if its score is a local maximum exceeding 1.2 times an exponential moving average of past scores (mixing 0.1, with no acquisition during the first 10 score evaluations); these are the ASE-uHAL default selector settings. After each acquisition, the configuration's descriptor is added to the design matrix of the committee shared by all particles, and the committee is resampled. Each iteration stops once 100 configurations have been acquired or the budget is exhausted; if a step overshoots, the highest-scoring configurations are kept. In practice, uHAL exhausted the budget in every iteration, acquiring 35--48 configurations per iteration (360--406 in total, compared with 1000 for the other samplers).

\textbf{Structure selection.} For the offline schemes, the 300 trajectory endpoints form the candidate set, and 100 configurations are selected by greedy farthest point sampling. Candidate descriptor vectors are $L^2$-normalized before pairwise distances are computed. When selecting structures at iteration $t > 1$, the descriptors of all previously selected structures are used to initialize the minimum-distance field. Newly selected structures are therefore maximally diverse relative to the entire accumulated dataset.

\textbf{Evaluation.} Energy error is the root mean square error (RMSE) of the per-atom potential energy over the 100 test structures. Force error is the RMSE over all $3N$ force components across the test set. Both are computed relative to the oracle (MACE-OFF-23-small). The test set consists of 100 structures obtained by constrained relaxation with the oracle on a $10\times10$ grid of $(\phi,\psi)$ values, holding the two dihedrals fixed and relaxing all other degrees of freedom.

\subsection{Si(111) surface reconstruction}
\label{app:exp_si}

Unless stated otherwise, all settings are as in Appendix~\ref{app:exp_mace}.

\textbf{System and model.} We fine-tune the MACE-MPA-0-medium foundation model~\cite{Batatia2025} on SCAN meta-GGA~\cite{sun_strongly_2015} DFT data computed with VASP~\cite{kresse_efficient_1996}, using settings consistent with a published Si--O dataset~\cite{erhard_modelling_2024}. The slab contains 28 atoms in a periodic $(2\times1)$ (111) cell with in-plane dimensions $7.73\times3.87$~\AA{} (surface area $A=25.89$~\AA$^2$) and a vacuum gap normal to the surface. The lowest 18 atoms are held fixed in all simulations and relaxations. Reference energies and forces are computed with VASP~\cite{kresse_efficient_1996} using the SCAN meta-GGA functional~\cite{sun_strongly_2015}. We use a plane-wave cutoff of 900~eV, a $\Gamma$-centred $k$-point mesh with spacing 0.23~\AA$^{-1}$ in the surface plane and a single $k$-point normal to the slab, Gaussian smearing of width 0.1~eV, no spin polarisation, and an electronic convergence threshold of $10^{-6}$~eV. These settings are consistent with a large previously published dataset for the silicon--oxygen system~\cite{erhard_modelling_2024}. A single-point calculation costs approximately 1.2 GPU-hours.

\textbf{Integrators and initialization.} All integrator settings are as for alanine dipeptide (\Cref{sec:mace}), except that SKMD and a-SKMD use a larger amplitude ($a=4.0$), which we found necessary for the higher descriptor dimension. Each sampler uses $J=30$ particles and a budget of 200 outer-loop steps of 500 MD steps. The $J=30$ particles are initialized from snapshots of a burn-in simulation of the unreconstructed slab.

\textbf{Training protocol.} Fine-tuning uses the same LoRA procedure and two-stage schedule, with three changes. The LoRA rank is increased to 16, the first-stage force weight is reduced to 100 and learning rates are reduced to 0.001 and 0.0001 for the two stages.

\textbf{SKMD outer-loop schedule.} Each run comprises 200 outer iterations with $a=4.0$. Iterations 1--160 are exploration with the median bandwidth. At iteration 161, $b$ is frozen as for alanine dipeptide. Over iterations 161--200, both $a$ and $b$ then decay linearly to half their values, allowing particles to settle into lower-energy regions of the potential energy landscape.

\textbf{a-SKMD acquisition schedule.} As for alanine dipeptide, with $a=4.0$: 50 exploration steps with the median bandwidth and no acquisition, after which $b$ and $\zeta_0$ are calibrated over the last 20 exploration steps and held fixed, with $\ell_0=250$. Each iteration stops once 30 configurations have been acquired or the 200-step budget is exhausted. In practice, a-SKMD acquired 30 configurations in every iteration, using on average 130 outer-loop steps.

\textbf{uHAL acquisition.} uHAL uses the same dynamics, committee and selection settings as for alanine dipeptide, stopping once 30 configurations have been acquired or the budget of 200 outer-loop steps is exhausted. In one of the three trials, the prior-only HAL score never exceeded the adaptive threshold during the first iteration, so no configurations were acquired and the trial was excluded.

\textbf{Structure selection.} For the offline schemes, 30 configurations per iteration are selected from the 200 trajectory endpoints by farthest point sampling, using the same descriptor normalization and accumulated-dataset initialization.

\textbf{Evaluation.} Rather than a test error, we target quantifying the error in a physical observable. The reconstruction energy per unit surface area is
\begin{equation}
    \label{eq:reconstruction}
    \Delta E_{\mathrm{recon}} = \frac{E_{\mathrm{reconstructed}} - E_{\mathrm{unreconstructed}}}{A},
\end{equation}
where each energy is that of the corresponding slab relaxed with the model being evaluated (or with DFT), keeping the lowest 18 atoms fixed. The reconstruction is stable when $\Delta E_{\mathrm{recon}}<0$. DFT gives $-0.279$~J\,m$^{-2}$, whereas the foundation model gives $+0.226$~J\,m$^{-2}$.

We report the error in $\Delta E_{\mathrm{recon}}$ relative to SCAN DFT. For each model, both the reconstructed and unreconstructed slabs are relaxed with that model, starting from geometries relaxed with the foundation model, using L-BFGS in ASE with a fixed cell and the lowest 18 atoms held fixed, until the maximum force on any atom is below 0.01~eV\,\AA$^{-1}$; the L-BFGS history is reset once if this is not reached within 500 steps. The DFT reference is obtained in the same way with VASP (conjugate-gradient ionic relaxation, same force threshold and fixed atoms, oracle settings as above), using the energy extrapolated to zero smearing. The DFT value is $\Delta E_{\mathrm{recon}} = -0.279$~J\,m$^{-2}$, whilst the initial foundation model value is $+0.226$~J\,m$^{-2}$.

\subsection{Compute infrastructure}
\label{app:compute}

\textbf{Neural network potential experiments.} These experiments were run on x86 server CPUs (AMD EPYC and Intel Xeon) as batch SLURM jobs, each with 2 CPU cores. The Langevin, SKMD and a-SKMD runs used 80~GB of memory and the UDD runs 200~GB. 

\Cref{tab:runtime_mb} reports the runtime and memory usage per active learning experiment, performed on a single x86 CPU. The runtime per training iteration and mean total runtime and memory (over 16 training iterations) are reported with the mean value over 50 random seeds. We observe that among the enhanced sampling schemes, SKMD and a-SKMD are more efficient compared to UDD. On average, a-SKMD had a lower runtime and memory compared to SKMD because for some trials, fewer simulation steps were performed as a result of adaptive stopping. UDD had the highest runtime since each training iteration requires cross-validation training of a committee of NN potentials, while simulation step requires force calls from all potentials in the committee. 

\begin{table}[ht]
    \centering
    \caption{\small{Runtime and memory usage for each active learning method in the M\"uller-Brown potential case study.}}
    \label{tab:runtime_mb}
    \begin{tabular}{lccc}
        \toprule
        Method & Runtime per training iteration & Total runtime & Memory \\
        \midrule
        Langevin & 00:00:55 & 00:14:42 & 16.7 GB \\
        UDD      & 00:03:46 & 01:00:12 & 23.0 GB \\
        SKMD     & 00:01:02 & 00:16:00 & 46.5 GB \\
        a-SKMD   & 00:00:48 & 00:12:57 & 35.9 GB \\
        \bottomrule
    \end{tabular}
\end{table}

\textbf{MACE experiments.} All alanine dipeptide and Si(111) active learning loops were run on the Avon cluster at the University of Warwick. Each sampling, labelling and fine-tuning job ran on a single NVIDIA Quadro RTX 6000 GPU (24~GB) with 4 CPU cores. MD and SKMD were run in LAMMPS~\cite{Thompson2022} with the Kokkos~\cite{johansson_lammps-kokkos_2025} backend and the symmetrix MACE evaluator~\cite{Kovacs2025}.

uHAL was run with ASE-uHAL~\cite{rocke2025aseuhal}. uHAL integrates the dynamics with ASE and evaluates MACE through its PyTorch calculator rather than through LAMMPS, which accounts for its longer run time in Table~\ref{tab:compute}.

\begin{table}[h]
\centering
\caption{Mean wall-clock sampling time per active learning iteration on one Quadro RTX 6000 GPU (range in parentheses), including burn-in and structure selection but excluding labelling and fine-tuning. Alanine dipeptide: 5 trials, iterations 2--10; Si(111): all iterations of the completed trials. Fine-tuning took a further 5--11 minutes per iteration for all methods.}
\label{tab:compute}
\begin{tabular}{lll}
\toprule
System & Method & Sampling time per iteration (min) \\
\midrule
Alanine dipeptide & a-SKMD & 12.7 (8.5--17.6) \\
Alanine dipeptide & uHAL   & 146 (138--157) \\
Si(111)           & a-SKMD & 26 (23--38) \\
Si(111)           & uHAL, iteration 1    & 187 (186--189) \\
Si(111)           & uHAL, iterations 2--5 & 17 (14--49) \\
\bottomrule
\end{tabular}
\end{table}

\textbf{DFT.} SCAN calculations for Si(111) were run with the GPU build of VASP on the same Quadro RTX 6000 nodes, at approximately 1.2~GPU-hours per structure, or about 36~GPU-hours per active learning iteration of 30 structures. Labelling therefore costs one to two orders of magnitude more than sampling, and dominates the total cost of the Si(111) study. The two DFT reference relaxations (16 and 24 ionic steps) were run with the CPU build of VASP on one 72-core node of the Raven cluster at the Max Planck Computing and Data Facility, taking 4.7 and 5.3 node-hours.

\textbf{Total usage.} The alanine dipeptide study (6 samplers $\times$ 5 trials $\times$ 10 iterations) used approximately 150~GPU-hours. The Si(111) study (5 samplers $\times$ 3 trials $\times$ 5 iterations) used approximately 60~GPU-hours of sampling and fine-tuning, approximately 3060~GPU-hours of DFT labelling, and approximately 720~CPU core-hours for the DFT reference relaxations.

\section{Ablation Studies}
\label{app:ablation}

To make it clearer how each component of SKMD drives performance gains, we perform extensive ablation studies which isolate the effect of each component in a) the sampling context, evaluating fidelity of the stationary distribution of the dynamics to the Boltzmann distribution, and b) the active learning context, evaluating the error in a NN potential trained with SKMD. 

\noindent \textbf{Sampling ablations.} We empirically study the effect of the time lag $\ell$, kernel magnitude ratio $\frac{a}{A}$ (for kernel magnitude $a=k(x,x))$, and particle set size $J$ on the invariant law of the dynamics. In the ablation of $\ell$, we implement a synchronous formulation of SKMD with $\ell = 0$, where each particle is updated simultaneously in each time step. In the ablation of $\frac{a}{A}$, we note that the kernel magnitude $a$ controls the relative weight of the SKMD biasing force while $A$ controls the relative weight of the gradient force. We measure the discrepancy of the empirical distribution of the trajectory over $10^5$ time steps with the Boltzmann distribution $\pi$ in terms of Wasserstein-2 ($W_2$) distance, computed using the Sinkhorn algorithm. The $W_2$ distance is reported in terms of the median value over 100 replications of the trajectory with a fixed initial condition and random realizations of Brownian motion. In ablations, the standardized values for non-ablated variables are $\ell=1$, $\frac{a}{A} = 2$, and $J = 32$. 

It is observed that numerically, synchronous SKMD with $\ell=0$ leads to higher discrepancy because of observed mode collapse, which is a well-known issue with SVGD-type dynamics that has motivated the use of the median heuristic \cite{Liu2016} for setting the kernel bandwidth. Introducing a small time lag, $\ell=1$, leads to lower discrepancy, whereas asynchronous SKMD with increasing $\ell$ leads to higher bias, as shown in \Cref{tab:samp_ablation} and \Cref{fig:boxplots}. The results also show that large values of $\frac{a}{A}$, where the SKMD biasing force dominates, leads to greater discrepancy, while smaller values of $\frac{a}{A}$, where the gradient force dominates, leads to lower discrepancy. With smaller $\frac{a}{A}$, the dynamics approach unbiased overdamped Langevin dynamics with re-initializations across the ensemble, ensuring good mixing of the trajectory. However, qualitatively, smaller $\frac{a}{A}$ also resulted in metastability and reduced coverage of the state space. Finally, it is observed that the discrepancy decreases for larger particle set sizes $J$. Smaller particle set sizes $J$ lead to higher discrepancy in part because the diffusion coefficient scales inversely with $J$ and causes the noise to dominate, such that the empirical distribution of the trajectory approaches uniform on a domain determined by the confining potential boundaries.

\begin{table}[ht]
    \centering
    \small
    \caption{\small{Wasserstein-2 distance between the Boltzmann distribution of the M\"uller--Brown potential and the invariant law of SKMD with varying parameters: time lag $l$, kernel magnitude ratio $a/A$, and particle set size $J$. Values are reported as the median over 100 random seeds. Lowest value per experiment is shown in bold.}}
    \label{tab:samp_ablation}
    \begin{tabular}{cc|cc|cc}
        \toprule
        $\ell$ & $W_2$ dist. &
        $\frac{a}{A}$ & $W_2$ dist. &
        $J$ & $W_2$ dist. \\
        \midrule
        0    & $8.30 \times 10^{-3}$ & 4    & $9.10 \times 10^{-3}$ & 8   & $7.90 \times 10^{-3}$ \\
        1    & $\mathbf{6.90 \times 10^{-3}}$ & 2    & $6.90 \times 10^{-3}$ & 16  & $9.80 \times 10^{-3}$ \\
        10   & $8.10 \times 10^{-3}$ & 1    & $3.70 \times 10^{-3}$ & 32  & $6.90 \times 10^{-3}$ \\
        100  & $8.30 \times 10^{-3}$ & 0.5  & $3.10 \times 10^{-3}$ & 64  & $3.80 \times 10^{-3}$ \\
        1000 & $8.80 \times 10^{-3}$ & 0.25 & $\mathbf{3.00 \times 10^{-3}}$ & 128 & $\mathbf{3.20 \times 10^{-3}}$ \\
        \bottomrule
    \end{tabular}
\end{table}

\begin{figure}[h]
    \centering
    \includegraphics[width=\textwidth]{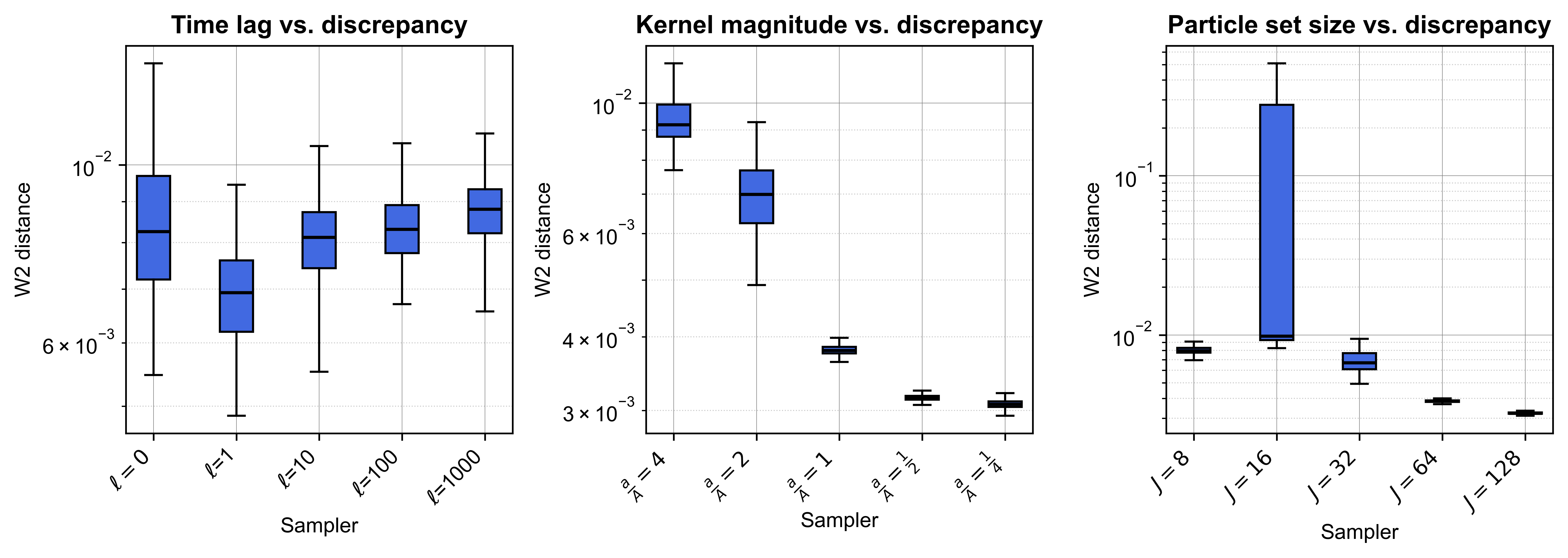}
    \caption{\small{Wasserstein-2 distance between the Boltzmann distribution of the M\"uller--Brown potential and the invariant law of SKMD with varying parameters: time lag $l$, kernel magnitude ratio $a/A$, and particle set size $J$.}}
    \label{fig:boxplots}
\end{figure}

\textbf{Active learning ablations.} We additionally perform ablation studies of SKMD in the context of active learning, in order to study the impact of the choice of kernel parameters. We implement SKMD using a Gaussian kernel with magnitude $a$ and bandwidth $b$. Since the adaptive stopping rule introduces a variable number of simulation steps per acquired datum, we remove variance from this effect by implementing SKMD without adaptive stopping, choosing data at fixed time intervals along the SKMD trajectory. In ablations, the standardized values for non-ablated variables are $\frac{a}{A} = 2$, and $b = 4.1 \times 10^{-1}$, which is equal to the median distance between particles in the initial dataset.

Despite the fact that large $\frac{a}{A}$ leads to bias in the invariant law of the dynamics (see \Cref{tab:samp_ablation} and \Cref{fig:boxplots}), larger values of $\frac{a}{A}$ in fact leads to lower RMSE of the trained NN potential in active learning, as seen in \Cref{fig:al_ablation_a}. This is due to the fact that a larger weight on the SKMD biasing force leads to increased repulsive dynamics which promote exploration of configuration space, which is advantageous for acquiring informative training data. \Cref{fig:al_ablation_b} shows that increasing the kernel bandwidth $b$ leads to stronger biasing forces, greater exploration of configuration space, and lower RMSE of the trained NN potential. We hypothesize that there exists an effective range of values for the kernel bandwidth which depends on the scale of the configuration space of a given system and the geometry of the kernel.

\begin{figure}[h]
    \centering
    \includegraphics[width=\textwidth]{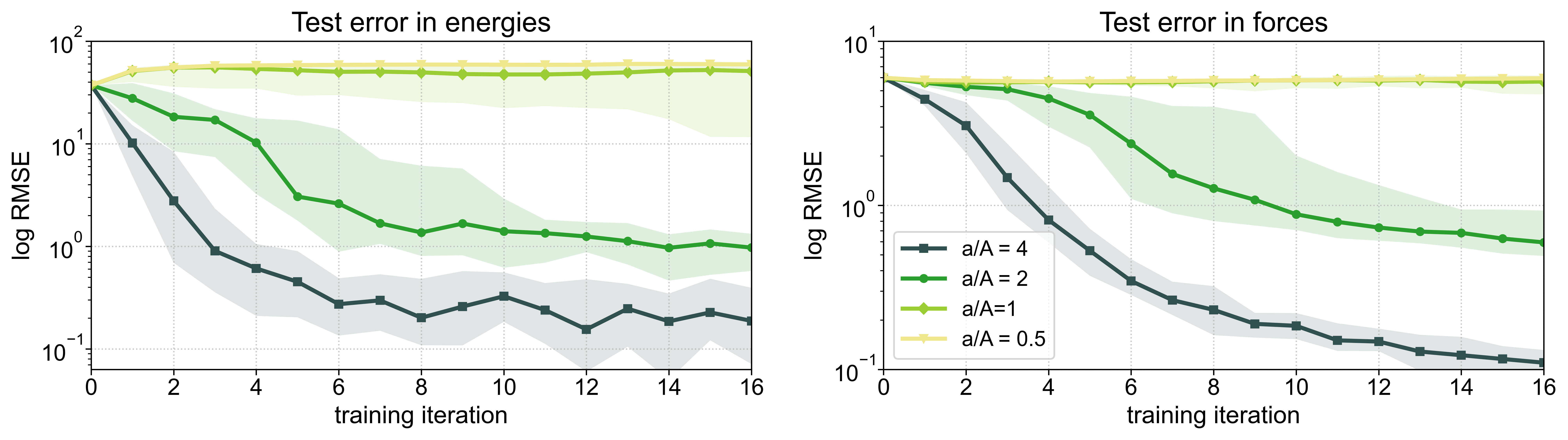}
    \caption{\small{Root mean square error (RMSE) in potential energy (left) and forces (right) of the NN potential relative to the M\"uller-Brown reference potential, using SKMD with varying kernel magnitude ratio $\frac{a}{A}$. The solid lines show the median error and the shaded regions show the 25th to 75th percentile range of the error across 50 trials.}}
    \label{fig:al_ablation_a}
\end{figure}

\begin{figure}[h]
    \centering
    \includegraphics[width=\textwidth]{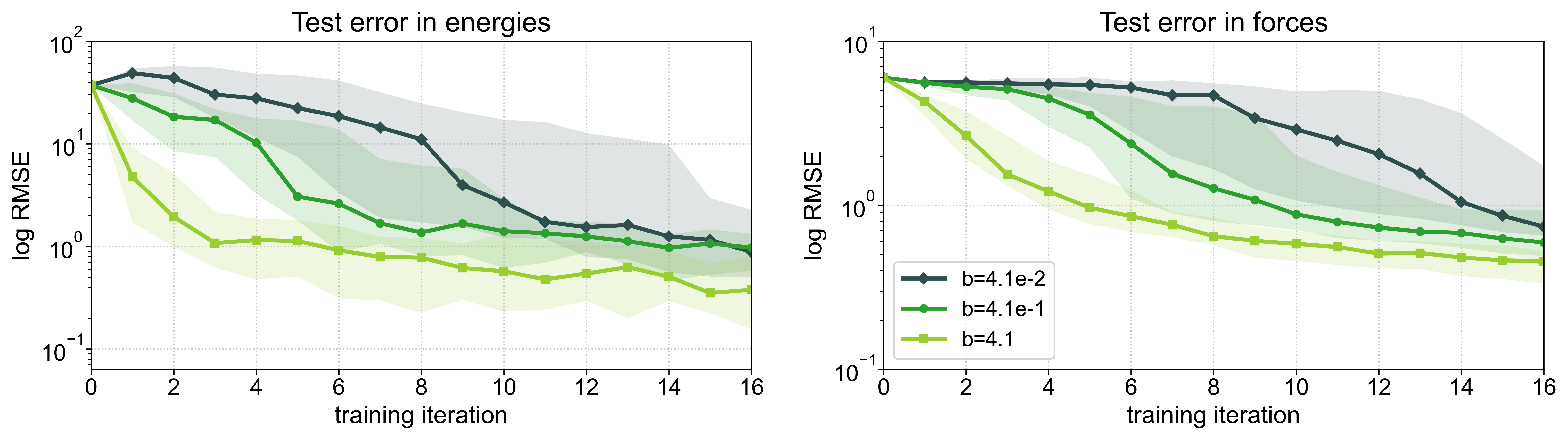}
    \caption{\small{Root mean square error (RMSE) in potential energy (left) and forces (right) of the NN potential relative to the M\"uller-Brown reference potential, using SKMD with varying kernel bandwidth parameter $b$. The solid lines show the median error and the shaded regions show the 25th to 75th percentile range of the error across 50 trials.}}
    \label{fig:al_ablation_b}
\end{figure}

The findings from our ablation studies are consistent with observations in active learning literature that wider coverage of state space leads to faster improved model performance but generally at the cost of reduced fidelity to the Boltzmann distribution. On the other hand, greater fidelity to the Boltzmann distribution can lead to metastability and difficulties with overcoming energy barriers \cite{Henin2022}. SKMD has the advantage of allowing for \textit{tunable} dynamical behavior, where the relation between parameters of the algorithm and the dynamical behavior -- fidelity-seeking vs. diversity-seeking -- is well understood and consistent with theory. This allows SKMD to be implemented in a staged scheme: an initial phase of exploratory dynamics spreads particles across configuration space to overcome metastability, followed by a secondary phase of exploitative dynamics where particles sample from key energy basins. As stated in our Conclusions, this two-phase scheme will be tested further in future work.

\section{Societal Impacts}
\label{app:impact}

We develop a method for improving the efficiency of atomistic simulations, with the potential to accelerate scientific discovery in chemistry, materials science, and drug design. By reducing the number of expensive quantum-mechanical calculations required for training MLIPs, our method can lower the computational cost and energy consumption of atomistic simulation workflows. While we do not foresee any direct negative societal impacts, atomistic simulation tools could be misused---for example, for the development of biochemical weapons. Such risks are not specific to our method. 

\section{Code and Data Availability}
\label{app:code}
The LAMMPS implementation of SKMD (\texttt{fix skmd}) is available at \url{https://github.com/Fraser-Birks/skmd-lammps}, and all experiments, figure scripts and active learning pipelines at \url{https://github.com/Fraser-Birks/skmd-examples}. The MACE descriptor computes are in a fork of symmetrix, \url{https://github.com/Fraser-Birks/symmetrix-public} (tag \texttt{skmd-v1.0}). Input data and per-run results are archived on Zenodo~\cite{zou_2026_24011550}.






\end{document}